\documentclass[sigconf]{acmart}
\usepackage[skip=3pt]{caption} 
\usepackage{afterpage} 

\usepackage{enumitem}
\usepackage{pifont}
\usepackage{float}
\usepackage{subcaption}
\usepackage{ragged2e}
\usepackage{tikz}
\usepackage{adjustbox}
\usepackage{graphicx}
\usepackage[most]{tcolorbox}
\usepackage{booktabs}
\usepackage{threeparttable}
\usepackage{makecell}
\usepackage{multirow}
\usepackage{array}
\usepackage{resizegather}
\usepackage{tcolorbox}
\usepackage[skip=2pt]{caption}
\usepackage{wasysym}
\usepackage{mathtools}
\usepackage{bm}
\usepackage[ruled,vlined,linesnumbered]{algorithm2e}
\SetKw{Or}{or}
\SetKw{True}{True}

\usepackage{amsthm}
\newtheorem{assumption}{Assumption}
\newtheorem{theorem}{Theorem}

\usepackage{color}

\newcommand{\myparagraph}[1]{\par\noindent\textbf{#1}.}

\usepackage{soul}

\usepackage{xspace}
\newcommand{\codename}{\textsc{Ghost}\xspace}

\newcommand{\argmin}{\operatorname{arg\,min}}

\newcommand{\ie}{\textit{i.e.}}

\acmYear{2026}\copyrightyear{2026}
\setcopyright{cc}
\setcctype[4.0]{by}
\acmConference[ASIA CCS '26]{ACM Asia Conference on Computer and Communications Security}{June 1--5, 2026}{Bangalore, India}
\acmBooktitle{ACM Asia Conference on Computer and Communications Security (ASIA CCS '26), June 1--5, 2026, Bangalore, India}
\acmDOI{10.1145/3779208.3785389}
\acmISBN{979-8-4007-2356-8/26/06}

\title{Mitigating Gradient Inversion Risks in Language Models via Token Obfuscation}

\author{Xinguo Feng}
\affiliation{%
  \institution{The University of Queensland}
  \institution{CSIRO's Data61}
  \streetaddress{St Lucia}
  \city{Brisbane}
  \country{Australia}
  }
\email{s.feng@uq.edu.au}

\author{Zhongkui Ma}
\affiliation{
  \institution{The University of Queensland}
  \streetaddress{St Lucia}
  \city{Brisbane}
  \country{Australia}}
\email{zhongkui.ma@uq.edu.au}

\author{Zihan Wang}
\affiliation{
  \institution{The University of Queensland}
  \institution{CSIRO's Data61}
  \streetaddress{St Lucia}
  \city{Brisbane}
  \country{Australia}}
\email{zihan.wang@uq.edu.au}

\author{Alsharif Abuadbba}
\affiliation{%
  \institution{CSIRO's Data61}
  \streetaddress{Marsfield}
  \city{Sydney}
  \country{Australia}}
\email{sharif.abuadbba@data61.csiro.au}

\author{Guangdong Bai}
\authornote{Corresponding author}
\affiliation{%
  \institution{City University of Hong Kong}
  \streetaddress{Hong Kong}
  \country{Hong Kong}}
\email{g.bai@cityu.edu.hk}

\begin{document}

\begin{abstract}
    Training and fine-tuning large-scale language models largely benefit from collaborative learning, but the approach has been proven vulnerable to gradient inversion attacks~(GIAs), which allow adversaries to reconstruct private training data from shared gradients. 
    Existing defenses mainly employ gradient perturbation techniques, \textit{e.g.}, noise injection or gradient pruning, to disrupt GIAs' direct mapping from gradient space to token space.     
    However, these methods often fall short due to the retention of semantics similarity across gradient, embedding, and token spaces. 
    Attackers can map proximate gradients into similar embeddings, and subsequently correspond them to tokens with similar semantics.
    
    In this work, we propose a novel defense mechanism named \codename (\textbf{\underline{g}}radient s\textbf{\underline{h}}ield with \textbf{\underline{o}}bfu\textbf{\underline{s}}cated \textbf{\underline{t}}okens), a \textit{token-level} obfuscation mechanism that neutralizes GIAs by decoupling the inherent connections across gradient, embedding, and token spaces. 
    \codename is built upon an important insight: due to the large scale of the token space, there exist semantically distinct yet embedding-proximate tokens that can serve as the \emph{shadow substitutes} of the original tokens, which enables a semantic disconnection in the token space while preserving the connection in the embedding and gradient spaces.
    \codename comprises a \emph{searching} step, which identifies semantically distinct candidate tokens using a multi-criteria searching process, and a \emph{selection} step, which selects optimal shadow tokens to ensure minimal disruption to features critical for training by preserving alignment with the internal outputs produced by original tokens. 
    Evaluation across diverse model architectures (from BERT to Llama) and datasets demonstrates the remarkable effectiveness of \codename in protecting privacy (as low as 1\% in recovery rate) and preserving utility (up to 0.92 in classification F1 and 5.45 in perplexity), in both classification and generation tasks against state-of-the-art GIAs and adaptive attack scenarios. \codename demonstrates a paradigm shift from the gradient level to the token level for effective mitigation against GIAs.
\end{abstract}

\ccsdesc[500]{Security and privacy}
\ccsdesc[500]{Computing methodologies~Machine learning}

\keywords{Collaborative Learning, Language Models, Gradient Inversion}

\maketitle

\section{Introduction}
Training and fine-tuning\footnote{In practice, training a language model typically refers to fine-tuning a published pre-trained language model with a local dataset for a downstream task. Without loss of generality, we interchangeably use ``training'' and ``fine-tuning'' to refer to this practice in this paper.} large-scale language models largely benefit from collaborative learning, attracting significant interest from both academia~\cite{ye2024fedllm} and industry~\cite{FedML}. However, this approach has been proven vulnerable to gradient inversion attacks~(GIAs)~\cite{dlg,tag,lamp,grab,dager,film,decep}, which allow adversaries to reconstruct private training data from shared gradients.
The growing capability of GIAs undermines the core promise of data privacy in collaborative learning~\cite{wang2025catch}, exposing participants to significant risks.

Efforts to mitigate GIAs for language models have led to the development of several defense mechanisms. One line of research aims to use lightweight methods to restrict gradient leakage by freezing the embedding layers to prevent direct token leakage~\cite{film}, or incorporating dropout layers to introduce randomness to the gradients~\cite{wei2020framework}. 
However, these methods are often bypassed by adaptive GIAs that abandon embedding layers or incorporate dropout learning.
Another line of research mainly employs gradient perturbation techniques, \textit{e.g.}, noise injection~\cite{dlg, wei2020framework}, or gradient pruning~\cite{dlg}. These methods aim to disrupt the direct mapping from gradient space to token space. While such defenses can partially defend against earlier attacks, recent advancements~\cite{grab, dager} demonstrate that they often fall short due to the retention of semantics similarity across gradient, embedding, and token spaces. 
This retention enables attackers to map proximate gradients into similar embeddings, and subsequently to tokens with similar semantics.
These limitations highlight the urgent need for a new perspective beyond the gradient space to effectively mitigate the privacy threat of GIAs.

\myparagraph{Our work}
We propose a novel defense mechanism, \codename (\textbf{\underline{g}}radient s\textbf{\underline{h}}ield with \textbf{\underline{o}}bfu\textbf{\underline{s}}cated \textbf{\underline{t}}okens), to mitigate gradient inversion risks in language model training within the collaborative learning setting. 
Unlike gradient-level defenses, \codename introduces a \textit{token-level} obfuscation mechanism that neutralizes GIAs by decoupling the inherent connections across gradient, embedding, and token spaces. 
This solution draws inspiration from a key insight into the vast scale of the token space.
Specifically, this allows for the existence of semantically distinct yet embedding-proximate tokens. 
These tokens can serve as shadow substitutes for the original tokens, enabling a semantic disconnection within the token space. 
At the same time, the embedding proximity preserves the effectiveness of fine-tuning in the embedding and gradient spaces.

\codename is designed as a two-step approach, including a \textit{searching} step and a \textit{selection} step. During the searching step, \codename applies a three-layer searching mechanism to narrow down neighbor tokens in the top-$k$ embedding similarity range that are likely to share distinct semantics with each original token as shadow tokens. Specifically, \codename excludes neighbor tokens that 1) have overlapping neighbors beyond a defined threshold, 2) are mutual neighbors of each other, and 3) share the same lemma (\textit{e.g.}, `\textit{went}'--`\textit{go}', `\textit{going}'--`\textit{go}'), with the original token. 
The shadow tokens serve as the candidates to replace the original tokens. 
During the selection step, \codename utilizes the shadow tokens to obfuscate the original data. 
Specifically, beam search is employed to iteratively identify the optimal shadow tokens that minimize the discrepancy between the internal outputs generated by the obfuscated data and those produced by the original data across all model layers. By systematically searching and selecting the optimal shadow tokens, \codename ensures that the obfuscated data possesses distinct semantics but retains the critical features required for effective model training. 

We evaluate \codename in terms of its defense performance and utility preservation performance.  
These experiments cover classification and generation models that span 6 model families (\textit{e.g.}, BERT, Llama), 21 models of varying sizes and types, and 6 diverse datasets. On its defense performance, we assess the attack efficacy of state-of-the-art GIA methods against \codename, using both comparative studies and adaptive scenarios. 
The results demonstrate \codename's remarkable protection capabilities, with the strongest attacks recovering fewer than 2\% of the original tokens. 
This surpasses the best baseline defense by defending up to 98\% more tokens. 
A human evaluation on the semantic similarity between original and obfuscated data confirms \codename's effectiveness in disrupting semantics.
In addition, we explore the robustness of \codename under potential adaptive attacks designed to bypass \codename, where the adversary has significantly enhanced capabilities with complete knowledge of the defense mechanisms and parameters. 
Results indicate that \codename remains highly effective, demonstrating its resilience to advanced adaptive attacks. 

On the utility preservation ability of \codename, we measure the model utilities across binary classification, multi-class classification, and sentence generation tasks with \codename applied. 
Results show F1 scores of up to 0.92 for classification models and perplexities as low as 5.45 for generative models, which are comparable to models fine-tuned on original data. This highlights \codename's effectiveness in utility preservation. In addition, for generative models, a human evaluation on the fluency and topic relevance of the original and defended models' generated content further confirms the utility preservation ability of \codename. Compared to other defenses, \codename achieves the best balance between data privacy and model utility. An ablation study is conducted to highlight the distinct and complementary roles of the components of \codename in achieving strong privacy protection and effective utility preservation.

\myparagraph{Contributions}
Our main contributions in this work are as follows:

\begin{itemize}[left=1em]
\item \textbf{A novel defense paradigm.} We introduce \codename, a pioneering \textit{token-level} defense mechanism that decouples the inherent connection between gradient, embedding, and token spaces exploited in GIAs. Unlike existing approaches that operate at the gradient level, \codename leverages the intrinsic properties of the three spaces to neutralize GIAs, setting a new direction for privacy-preserving collaborative learning.

\item \textbf{A generic and robust obfuscation process.} \codename incorporates a two-stage obfuscation process, which identifies semantically distinct but embedding-proximate tokens as shadow substitutes for the original tokens, achieving robust obfuscation while maintaining the required critical features for effective model training.

\item \textbf{A formal framework and theoretical analysis.} We establish a strong formal foundation for \codename, leveraging optimization theory and embedding space analysis to demonstrate utility preservation through bounded loss differences and defense efficacy by disrupting semantic alignment in token reconstruction.

\item \textbf{Comprehensive empirical evaluation.} We conduct extensive experiments across diverse model architectures (\textit{e.g.}, BERT, Llama) and datasets. Results highlight \codename's superior defense capability against state-of-the-art GIAs and challenging adaptive attacks while providing the best balance for model utilities. Human evaluations are provided to validate \codename's effectiveness in preserving utility and protecting privacy. An ablation study is conducted to assess the contributions of its core components.

\end{itemize}

\myparagraph{Availability} Our code is publicly available at: \url{https://github.com/Trusted-System-Lab/GHOST}.

\section{Background}
This section introduces the necessary background knowledge to understand \codename comprehensively.

\subsection{Language Modeling}
Language modeling has become a major focus in NLP due to the development of neural networks. There are two main frameworks, Masked Language Models (MLM), \textit{e.g.}, BERT~\cite{bert}, RoBERTa~\cite{roberta}, DeBERTa~\cite{he2020deberta}, and Causal Language Models (CLM), \textit{e.g.}, GPT-2~\cite{gpt2}, Llama~\cite{touvron2023llama}, Gemma~\cite{team2024gemma}. These models adopt the transformer encoder or decoder architectures~\cite{transformer} and possess millions to billions of parameters pre-trained on a large corpus of text. Additional layers are attached to these pre-trained models for specific downstream tasks if necessary, such as sentiment analysis~\cite{sentiment}, named-entity recognition~\cite{ner}, and machine translation~\cite{koehn2009statistical}.

Recently, training or fine-tuning language models largely benefit from collaborative learning methods such as Federated Learning (FL)~\cite{fl, FedML, tian2022fedbert, cai2023federated}, which enable multiple clients to collaboratively train a model while preserving the privacy of their individual data. Instead of sharing raw data, clients transmit training information, such as gradients, to a central server for coordination and aggregation.

\subsection{Tokens and Embeddings}
\label{approach_preliminaries}
We introduce the tokens and embeddings in language models.
The following focuses on a single input sentence processed by the language models, which can be extended to batched input sentences without the loss of generality.

\myparagraph{Tokens in embedding space}
A \textit{token} is the smallest unit of text processed by a language model. A token \textit{embedding} is a vector representation of the token that captures the semantics of the token in a high-dimensional embedding space. The token embedding layer is typically trained in the pre-training phase and undergoes a minor update during the fine-tuning phase. Consider a sentence comprising $n_w$ words, denoted as $\bm{w} = (w_1, w_2, \cdots, w_{n_w})$. A tokenizer processes this sentence and converts it into a sequence of $n_t$ tokens, $\bm{t} = (t_1, t_2, \cdots, t_{n_t})$.
Each token $t_i$ in the sequence $\bm{t}$ is an element of a predefined vocabulary $\mathcal{V}$. Notably, the number of words and the number of tokens are typically not equal, \textit{i.e.}, $n_w \neq n_t$.
The token sequence is mapped to a corresponding embedding matrix $\bm{E} \in \mathbb{R}^{n_t \times d_e}$, where $d_e$ is the dimensionality of the embedding space.

\myparagraph{Token neighbors based on embedding}
Given the embedding vectors $\bm{e}'$ and $\bm{e}''$ of any two token $t', t'' \in \mathcal{V}$, the cosine distance is used to measure their distance, \textit{i.e.}, 
$$
d(t', t'') \triangleq 1 - \frac{\bm{e}' (\bm{e}'')^T}{\|\bm{e}'\| \|\bm{e}''\|}.
$$
The \textit{top-$k$ neighbors} of a token $t$ are the $k$ tokens having the smallest $k$ embedding distances, denoted as
$$
\mathcal{N}(t, k) \triangleq \{
t' \mid t' \in \mathrm{argsort}(d(t, \mathcal{V}))[1:k]
\},
$$
where $d(t, \mathcal{V})$ is a vector representing the distance between $t$ and each token in $\mathcal{V}$ (including $t$ itself).
The range $[1:k]$ excludes the first token, which is $t$ itself. 

\myparagraph{Lemma of a token}
The \textit{lemma} of a token represents its base linguistic form.
For a full-word token, the lemma is its dictionary form, without inflectional endings or grammatical variations.
If the token is a subword, \textit{e.g.}, a fragment of a word, its lemma is simply the token itself because subwords do not represent full linguistic units. 
Let $\ell$ denote the lemmatization function that maps a token $t$ to its lemma $l$,
$$
\ell: \mathcal{V} \to \mathcal{L}, \quad \ell(t) = l.
$$

\subsection{Gradient Inversion Attacks}
Gradient inversion attacks~\cite{dlg, tag, idlg, ig, decep, film, lamp, grab, dager} are techniques designed to reconstruct a client's training data by exploiting gradient information exchanged during FL. Unlike privacy and security concerns against centralized learning~\cite{wang2024corelocker, wang2025re, liu2024purpose}, 
these attacks assume that an adversary can intercept the communication between the server and the victim client with full knowledge of the model architecture and parameters. To recover the input data batch $\bm{B}^*_i$, an alternative data batch $\bm{B}_{i}$ is randomly initialized and optimized by solving the following optimization problem, 
\begin{equation*} \small
    \underset{\bm{B}_{i}}{\argmin} ~\mathcal{D}(\nabla_{\bm{\theta}^t} \mathcal{L}(\bm{B}_{i}),~ \nabla_{\bm{\theta}^t} \mathcal{L}(\bm{B}^*_i)), 
\end{equation*} 
where $\nabla_{\bm{\theta}^t} \mathcal{L}(\bm{B}_{i})$ is the dummy gradients, $\nabla_{\bm{\theta}^t} \mathcal{L}(\bm{B}^*_i)$ is the observed gradients, and $\mathcal{D}$ is a distance measure function.
For text data, this alternative batch represents the embeddings of the data, which are then mapped back into discrete tokens by matching them to the token embeddings of the language model. Recent advanced attacks~\cite{lamp, grab, dager} improve this process by using discrete optimization techniques or analyzing the rank of attention matrices, resulting in more accurate or even near-perfect recovery of the original data.

\section{Approach}
\label{approach}
This section provides a detailed description of \codename against GIAs on language models. 
Section~\ref{threat_model}  discusses the threat model, outlining the adversarial assumptions and attack scenarios. 
Section~\ref{overview} provides a high-level overview of the overall approach. Sections~\ref{neighbor_filtering} and \ref{neighbor_optimization} delve into the specifics of two components of \codename, explaining their roles in ensuring data privacy protection while preserving model utility.

\subsection{Threat Model}
\label{threat_model}
In this work, we consider the FedSGD~\cite{fl} setting, which involves sharing gradients after each local training iteration.
We consider an honest-but-curious adversary that either controls the server or eavesdrops on the communication between the victim client and the server. 
The adversary has full knowledge of the language model architecture and parameters. Its objective is to recover the private training data of the victim client by utilizing its shared gradient updates. 
Our threat model aligns with the ones in current attacks, and reflects a realistic training scenario where the only interaction between the client and the server is the exchange of gradient updates.

Some existing attacks~\cite{dlg, tag, lamp, film} operate under a threat model that assumes additional knowledge, such as the ground truth labels or the exact lengths of individual sequences within the data batch. While such assumptions may not always be realistic in practical training settings, our defined threat model does not explicitly exclude this information to avoid providing an undue advantage to the defender.

\subsection{Overview of \codename}
\label{overview}
Figure~\ref{fig:overview} shows an overview of \codename, which employs a two-stage process: 
\begin{enumerate}[left=0pt]
    \item a \textit{searching} step that refines the pool of shadow tokens with distinct semantics using well-defined similarity heuristics, and
    \item a \textit{selection} step that identifies the optimal shadow tokens that preserve critical features (hidden states) of the original data.  
\end{enumerate}

\begin{figure*}[!t]
    \centering
    \includegraphics[width=0.9\linewidth]{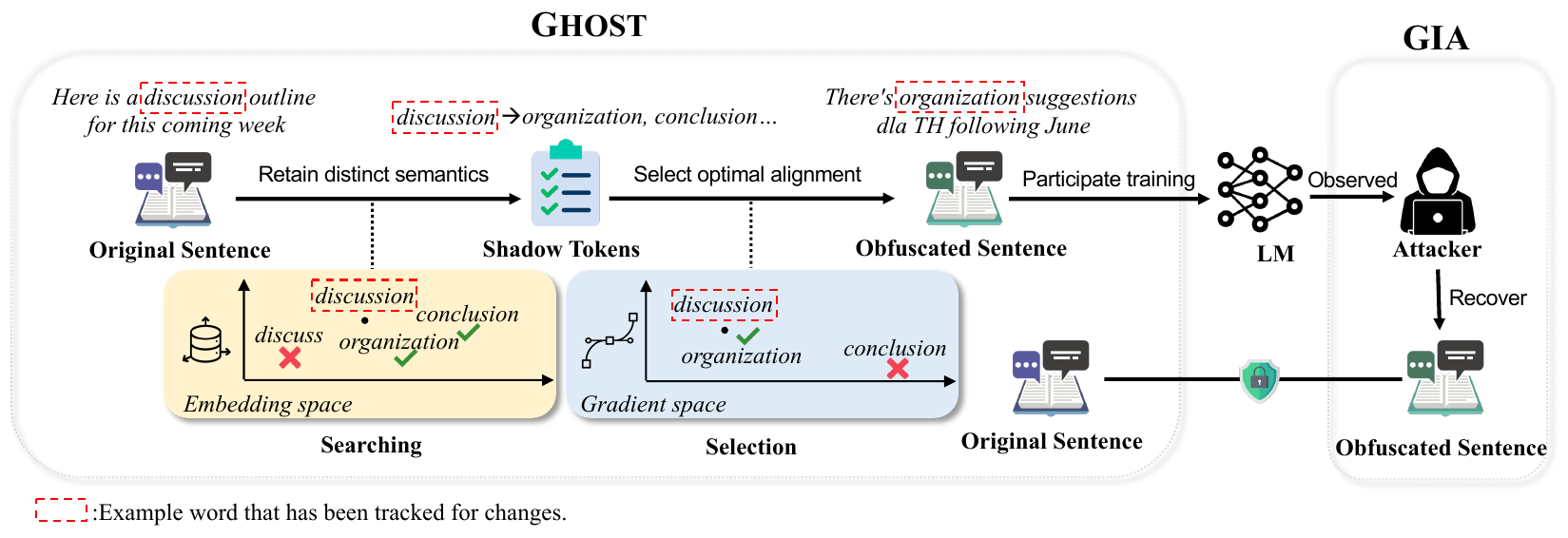}
    \caption{Overview of \codename.}
    \label{fig:overview}
\end{figure*}
By combining these two steps, \codename ensures that the obfuscated data maintains distinct semantics from the original data to defend against GIAs while preserving the critical features necessary for effective model training. We highlight that \codename operates without relying on gradient information during the fine-tuning, allowing it to obfuscate data prior to the fine-tuning process. 

\subsection{Searching}
\label{neighbor_filtering}

GIAs for language models aim to recover the original tokens by optimizing dummy embeddings that produce similar gradients in the gradient space, then map back to the token space. 
The objective of the searching step in \codename is to break this connection by retaining the semantically distinct neighbors of the original token as shadow tokens. The shadow tokens serve as replacement candidates for the original tokens.
This is accomplished through the application of three defined similarity measures:
\begin{itemize}[left=0pt]
\item \textbf{Indirect Similarity}:
Tokens $t'$ and $t''$ satisfy indirect similarity if their sets of top-$k$ neighbors have an overlap larger than a pre-defined threshold:
$$
S_i(t', t'') \triangleq |\mathcal{N}(t', k) \cap \mathcal{N}(t'', k)| > \tau_o.
$$

\item \textbf{Direct Similarity}:
Tokens $t'$ and $t''$ satisfy direct similarity if they are in each other's top-$k$ neighbors:
$$
S_d(t', t'') \triangleq t' \in \mathcal{N}(t'', k) \wedge t'' \in \mathcal{N}(t', k).
$$

\item \textbf{Common Lemma Similarity}: Tokens $t'$ and $t''$ satisfy common lemma similarity if they share a common lemma:
$$
S_c(t', t'') \triangleq (\ell(t') = \ell(t'')).
$$
\end{itemize}

Therefore, given a token $t$ and its top-$k$ neighbors $\mathcal{N}(t, k)$, for every neighbor token $t'$ in $\mathcal{N}(t, k)$, it is considered as a shadow token that forms the shadow token set $\mathcal{S}_t$ if it does not satisfy any one of the similarities:
$$
\mathcal{S}_t = \{t' \mid (t' \nvDash S_i(t',t) \vee S_d(t',t) \vee S_c(t',t)), \forall t' \in \mathcal{N}(t, k)\}.
$$

The searching process is detailed in Algorithm~\ref{alg:neighbor_filter}.
Each token in the vocabulary undergoes a similarity evaluation with the above-mentioned similarity criteria against its top-$k$ neighbors. The neighbor tokens that satisfy the similarity criteria are removed, as they are considered to have similar semantics.
For each token, if it results in an empty candidate set, top-$k$ is gradually incremented for a broader search range.
This structured process results in a refined set of shadow tokens $\mathcal{S}$, which is subsequently passed to the selection step.

\begin{algorithm}[t]
\caption{search($\cdot$)}
\label{alg:neighbor_filter}
\KwIn{
    $\mathcal{V}$ - the vocabulary of the language model; \\
    top-$k$ - the number of neighbors;\\
    $\tau_o$ - the neighbor overlap threshold.
}
\KwOut{
    the shadow tokens.
}

$\mathcal{S} \gets \mathrm{Dict}()$\;
$k \gets \text{top-}k$\;

\ForEach{$t \in \mathcal{V}$}{
\While{\True}{
    $\mathcal{S}_t \gets \mathcal{N}(t, k)$\;
    $\mathcal{S}_t' \gets \mathrm{copy}(\mathcal{S}_t)$\;
    $l_t \gets \ell(t)$\;

    \ForEach{$t' \in \mathcal{S}_t$}{
        $\mathcal{S}_{t'} \gets \mathcal{N}(t', k)$\;
        $l_{t'} \gets \ell(t')$\;
        
        \If{
        $(|\mathcal{S}_t \cap \mathcal{S}_{t'}| > \tau_o$
        \Or $t \in \mathcal{S}_{t'}$
        \Or $l_t = l_{t'})$
        }{
            $\mathcal{S}_t'.\mathrm{remove}(t')$\;
        }
    }
\If{$\mathcal{S}_t' \neq \varnothing$}{
    $\mathcal{S}.\mathrm{set}(t, \mathcal{S}_t')$\;
    $k \gets \text{top-}k$\;
    \textbf{break}\;
}
\Else{$k \gets k + 10$\;}
}
}
\Return $\mathcal{S}$\;

\end{algorithm}

\subsection{Selection}
\label{neighbor_optimization}

The selection step aims to identify the optimal shadow tokens to create obfuscated data, which preserves the critical features for effective model training.  
The selection objective is to find shadow tokens such that the hidden states produced by the obfuscated data closely align with those of the original data across all layers of the language model.
The selection minimizes the Mean Squared Error (MSE) between the hidden states. 

Specifically, given an input token sequence $\bm{t}$, the goal is to identify the optimal shadow token sequence $\bm{t}'$ from the shadow token set $\mathcal{S}$ for each token $t_i$ in $\bm{t}$:
$$
\bm{t}' = \underset{\bm{t}' \in \mathcal{S}}{\argmin} \sum_{i=1}^{n_l} \text{MSE}(\bm{h}^{(i)}(\bm{t}), \bm{h}^{(i)}(\bm{t}')),
$$
where $\bm{h}^{(i)}(\bm{t})$ and $\bm{h}^{(i)}(\bm{t}')$ represents the hidden states taking $\bm{t}$ and $\bm{t}'$ as input, respectively.
Here, $n_l$ is the number of considered layers in the model.

The selection process is conducted iteratively for each token in the input sequence, as outlined in Algorithm~\ref{alg:neighbor_optimization}.
First, the original text sentence $\bm{x}$ is tokenized into a token sequence $\bm{t}_{ori}$ (line 1). The hidden states of the original token sequence $\mathcal{H}_{ori}$ are computed (line 2).
Then, the shadow tokens are used to randomly initialize an initial obfuscated token sequence $\bm{t}$ (line 3).
Next, each token $t$ in $\bm{t}$ is substituted with each shadow token in $\mathcal{S}_{t_{ori}}$. The optimization objective is to minimize the MSE of the two hidden states. The shadow tokens that produce hidden states closest to the original hidden states $\mathcal{H}_{\bm{t}}$ are retained (lines 6--9). If the MSE change between two successive iterations is less than a pre-defined threshold $\tau_d$, early stopping is triggered (lines 10--12). Finally, the obfuscated token sequence $\bm{t}$ is detokenized and returned as the obfuscated textual sentence $\bm{x}$.





\begin{algorithm}[t]
\caption{select($\cdot$)}
\label{alg:neighbor_optimization}
\KwIn{
    $\bm{x}$ - the original text sentence; \\
    $\bm{f}$ - the language model; \\
    $\mathcal{S}$ - the shadow tokens; \\
    $\tau_d$ - the hidden states distance threshold.
}
\KwOut{
    the obfuscated text sentence.
}

$\bm{t}_{\mathrm{ori}} \gets \mathrm{tokenize}(\bm{x})$\;
$\mathcal{H}_{\mathrm{ori}} \gets \mathrm{getHStates}(\bm{t}_{\mathrm{ori}}, \bm{f})$\;
$\bm{t} \gets \mathrm{randomInitialize}(\mathcal{S}, \bm{t}_{\mathrm{ori}})$\;
$\mathcal{H} \gets \mathrm{getHStates}(\bm{t}, \bm{f})$\;
\While{\True}{
    \ForEach{$i \in \mathrm{range}(0,\bm{t}.\mathrm{length})$}{
        $\mathcal{S}_{t_{ori}} \gets \mathcal{S}.\mathrm{get}(\bm{t}_{ori}[i])$\;
        $
        \bm{t}[i] \gets 
        \underset{
        \bm{t}'[i] \in \mathcal{S}_{t_{ori}}
        }{
        \argmin
        }
        \mathrm{MSE}(
        \mathcal{H}_{\mathrm{ori}},
        \mathrm{getHStates}(\bm{t}', \bm{f})
        )
        $\;
        $\mathcal{H}' \gets \mathrm{getHStates}(\bm{t}, \bm{f})$\;
    }
    \If{
    $| \mathrm{MSE}(\mathcal{H}_{\mathrm{ori}}, \mathcal{H}) - \mathrm{MSE}(\mathcal{H}_{\mathrm{ori}}, \mathcal{H}') | < \tau_d$
    }{
        \textbf{break}\;
    }
    $\mathcal{H} \gets \mathcal{H}'$\;
}
$\bm{x} \gets \mathrm{detokenize}(\bm{t})$\;
\Return $\bm{x}$\;

\end{algorithm}

\section{Theoretical Analysis}

This section analyzes the effectiveness of \codename. We demonstrate that the model fine-tuned on the dataset $\widetilde{\mathcal{D}}$ obfuscated by \codename achieves similar performance on the original dataset $\mathcal{D}$, but with different parameter gradients to defend against GIA.
The key insight is that the forward behavior of a model, \textit{e.g.}, its input-output mapping, is generally more stable than its backward behavior, \textit{e.g.}, the gradients, during fine-tuning. As a result, preserving the stability of input-output pairs ensures that the model maintains its utility, whereas the instability in the gradients makes it more difficult for GIAs to succeed. Furthermore, fine-tuning a pre-trained model enhances the local stability of inputs in the embedding space. This characteristic is beneficial when substituting original tokens with shadow tokens, as it helps preserve the model’s utility while disrupting the gradient-token correspondence exploited by GIAs.

This section combines the model and loss function into one output function $L(\bm{x}; \bm{\theta})$. For simplicity, we omit the label term $y$.
This is reasonable as a loss function with its scalar output is usually used to measure the model's performance to guide the training.

We discuss three models with different parameter settings: 
Given a pre-trained model $L(\bm{x}; \bm{\theta})$ and a dataset $\mathcal{D}$, we denote the fine-tuned model as $L(\bm{x}; \bm{\theta}^*)$ with the parameters $\bm{\theta}^*$.
Additionally, we denote $L(\bm{x}; \tilde{\bm{\theta}})$ as the model fine-tuned on the dataset $\widetilde{\mathcal{D}}$ obfuscated by \codename.

\subsection{Assumptions}

With the pre-trained model $L(\bm{x}; \bm{\theta})$, we assume that the model $L(\bm{x}; \bm{\theta}^*)$ fine-tuned on $\mathcal{D}$ achieves a low loss value on $\mathcal{D}$, indicating good performance.
Similarly, we aim to make the model $L(\bm{x}; \tilde{\bm{\theta}})$ fine-tuned on $\widetilde{\mathcal{D}}$ achieve a low loss value on $\mathcal{D}$, reflecting utility preservation.
Formally, we set the following assumptions.

As shown by Roux et al.~\cite{le2007continuous}, differentiability is a fundamental property of neural networks. It essentially allows parameters of neural networks to be optimized via gradients. Based on this, we set the following assumption.

\begin{assumption}[Model and loss function]
    The model $L(\bm{x}; \bm{\theta})$ is a continuous and differentiable function with respect to $\bm{x}$ and $\bm{\theta}$ and outputs a non-negative value.
\end{assumption}

Naturally, a model fine-tuned on a specific dataset $\mathcal{D}$ (or $\widetilde{\mathcal{D}}$) tends to exhibit lower loss values on inputs $\bm{x} \in \mathcal{D}$ (or $\tilde{\bm{x}} \in \widetilde{\mathcal{D}}$). Moreover, the gradient of the loss with respect to each parameter is typically smaller at these inputs, as the training process drives the parameters towards local optima in the loss landscape. 
We formulate this as the following assumption. 

\begin{assumption}[Dataset adaptation]
\label{assumption:dataset_adaptation}
    The model fine-tuned on a specific dataset has a lower loss value on the same dataset than on other datasets.
    That is, for any $\bm{x} \in \mathcal{D}$ and $\tilde{\bm{x}} \in \widetilde{\mathcal{D}}$, we have
    \begin{gather*}
    \begin{aligned}
        L(\bm{x}; \bm{\theta}^*) \leq L(\tilde{\bm{x}}; \bm{\theta}^*), & \quad
        L(\tilde{\bm{x}}; \tilde{\bm{\theta}}) \leq L(\bm{x}; \tilde{\bm{\theta}}), \\
        \nabla_{\bm{\theta}^*} L(\bm{x}; \bm{\theta}^*) \leq \nabla_{{\bm{\theta}^*}} L(\tilde{\bm{x}}; \bm{\theta}^*), & \quad
        \nabla_{\tilde{\bm{\theta}}} L(\tilde{\bm{x}}; \tilde{\bm{\theta}}) \leq \nabla_{\tilde{\bm{\theta}}} L(\bm{x}; \tilde{\bm{\theta}}),
    \end{aligned}
    \end{gather*}
    where we operate the inequality operators on a gradient vector to represent the inequality relation of each corresponding gradient component in two vectors.
\end{assumption}

As shown by Radiya et al.~\cite{radiya2020fine}, fine-tuning pre-trained models typically introduces only small parameter updates. We formulate this as  the following assumption.
\begin{assumption}[Small parameter updates in fine-tuning]
\label{assumption:small_update}
    The fine-tuning process results in small parameter updates to adapt the model to the new dataset, \ie,
    
    \begin{gather*}
        \| \Delta_{\bm{\theta}^*} \| \leq \epsilon_{\bm{\theta}^*}, 
        \quad
        \| \Delta_{\tilde{\bm{\theta}}} \| \leq \epsilon_{\tilde{\bm{\theta}}},
    \end{gather*}
    where $\epsilon_{\bm{\theta}^*}$ and $\epsilon_{\tilde{\bm{\theta}}}$ are a small positive value and $\| \cdot \|$ is the Euclidean norm.
\end{assumption}

As shown by Zhang et al.~\cite{zhang2018efficient}, neural networks exhibit the local smoothness property that perturbations within a certain range yield similar output, hence similar gradients. 
Our approach follows this property to construct the obfuscated data, where it systematically selects tokens within a close embedding range that yield similar outputs but exhibit distinct semantics. 
We set the following property to indicate this.

\begin{assumption}[Similar outputs on obfuscated data]
\label{assumption:dataset_modification}
    The perturbation introduced by neighbor tokens is small, such that the produced loss and gradients are similar, satisfying the local smoothness property.
    Formally, for any $\bm{x} \in \mathcal{D}$ and $\tilde{\bm{x}} \in \widetilde{\mathcal{D}}$,
    \begin{gather*}
        | L(\bm{x}; \bm{\theta}) - L(\tilde{\bm{x}}; \bm{\theta}) | \leq \epsilon_{L_{\bm{\theta}}}, 
        \quad
        \| \nabla_{\bm{\theta}} L(\bm{x}; \bm{\theta}) - \nabla_{\bm{\theta}} L(\tilde{\bm{x}}; \bm{\theta}) \| \leq \epsilon_{g_{\bm{\theta}}},
    \end{gather*}
    where $\epsilon_{L_{\bm{\theta}}}$ and $\epsilon_{g_{\bm{\theta}}}$ are small positive values.
\end{assumption}

For all small positive values in above assumptions, \ie, $\epsilon_{\bm{\theta}^*}$, $\epsilon_{\tilde{\bm{\theta}}}$, $\epsilon_{L_{\bm{\theta}}}$, and $\epsilon_{g_{\bm{\theta}}}$, a joint upper bound $\epsilon$ exists, which can be defined as $\epsilon = \text{max}(\epsilon_{\bm{\theta}^*}, \epsilon_{\tilde{\bm{\theta}}}, \epsilon_{L_{\bm{\theta}}}, \epsilon_{g_{\bm{\theta}}})$. In Section~\ref{exp: validation_theory}, we experimentally validate the existence of this joint upper bound.

\subsection{Model Utility Preservation and Defense against GIAs}

This section analyzes the properties of the model fine-tuned with the obfuscated dataset $\widetilde{\mathcal{D}}$ produced by \codename, where it replaces original tokens with selected neighbor tokens to preserve utility while defending against GIAs. We discuss the model properties from two aspects, including model utility preservation (indicated by loss values) and defense against GIAs (indicated by the mapping from input data points to gradients). In the following analysis, we use the limit $\epsilon \to 0$ only to express asymptotic scaling. We note that our defense does not aim for $\epsilon$ to vanish, which will cause the obfuscated and original tokens to be identical, and the defense will provide no protection.

\subsubsection{Model Utility Preservation}

The following theorem demonstrates that for data point $\bm{x}$ and its obfuscated one $\tilde{\bm{x}}$, the loss value difference between $L(\bm{x}, \tilde{\bm{\theta}})$ and $L(\tilde{\bm{x}}, \tilde{\bm{\theta}})$ has a linear asymptotic growth rate as $\epsilon \to 0$.

\begin{theorem}[Model utility preservation]
\label{thm:utility}

    Given the pre-trained model $L(\bm{x}; \bm{\theta})$ and the model $L(\bm{x}; \tilde{\bm{\theta}})$ fine-tuned on the obfuscated dataset $\widetilde{\mathcal{D}}$, for any pair of the original data point $\bm{x} \in \mathcal{D}$ and its obfuscated data point $\tilde{\bm{x}} \in \widetilde{\mathcal{D}}$, their loss value difference satisfies the following linear upper asymptotic growth rate, when $\epsilon \to 0$,
    \begin{gather*}
        L(\bm{x}; \tilde{\bm{\theta}}) - L(\tilde{\bm{x}}; \tilde{\bm{\theta}}) = O(\epsilon),
    \end{gather*}
    where $\epsilon$ is defined by Assumption~\ref{assumption:small_update}.
    More strictly, we have
    \begin{gather*}
        0
        \leq \limsup_{\epsilon \to 0} \frac{L(\bm{x}; \tilde{\bm{\theta}}) - L(\tilde{\bm{x}}; \tilde{\bm{\theta}})}{\epsilon}
        < k,
    \end{gather*}
    where $k$ is a positive constant.
\end{theorem}

\begin{proof}

    Given the Taylor expansions of the loss function, which describe the model state from the pre-trained model to the fine-tuned model on the obfuscated dataset, we have
    \begin{gather*}
    \begin{aligned}
        L(\bm{x}, \tilde{\bm{\theta}})
        &=
        L(\bm{x}; \bm{\theta}) + \nabla_{\bm{\theta}} L(\bm{x}; \bm{\theta}) \Delta_{\tilde{\bm{\theta}}} + O(\| \Delta_{\tilde{\bm{\theta}}} \|^2),\\
        L(\tilde{\bm{x}}, \tilde{\bm{\theta}})
        &=
        L(\tilde{\bm{x}}; \bm{\theta}) + \nabla_{\bm{\theta}} L(\tilde{\bm{x}}; \bm{\theta}) \Delta_{\tilde{\bm{\theta}}} + O(\| \Delta_{\tilde{\bm{\theta}}} \|^2). \\
    \end{aligned}
    \end{gather*}
    By Assumptions~\ref{assumption:small_update} and \ref{assumption:dataset_modification}, we have
    \begin{align*}
        & L(\bm{x}; \tilde{\bm{\theta}}) - L(\tilde{\bm{x}}; \tilde{\bm{\theta}}) \\
        \leq & | L(\bm{x}, \tilde{\bm{\theta}}) - L(\tilde{\bm{x}}, \tilde{\bm{\theta}}) | \\
        = &  
        | L(\bm{x}; \bm{\theta}) - L(\tilde{\bm{x}}; \bm{\theta}) + \nabla_{\bm{\theta}} L(\bm{x}; \bm{\theta}) \Delta_{\tilde{\bm{\theta}}} - \nabla_{\bm{\theta}} L(\tilde{\bm{x}}; \bm{\theta}) \Delta_{\tilde{\bm{\theta}}} + \\
        & O(\| \Delta_{\tilde{\bm{\theta}}} \|^2) - O(\| \Delta_{\tilde{\bm{\theta}}} \|^2) | \\
        \leq & | L(\bm{x}; \bm{\theta}) - L(\tilde{\bm{x}}; \bm{\theta}) | + \| \nabla_{\bm{\theta}} L(\bm{x}; \bm{\theta}) - \nabla_{\bm{\theta}} L(\tilde{\bm{x}}; \bm{\theta}) \| \cdot \| \Delta_{\tilde{\bm{\theta}}} \| + \\
        & | O(\| \Delta_{\tilde{\bm{\theta}}} \|^2) - O(\| \Delta_{\tilde{\bm{\theta}}} \|^2) | \\
        \leq & \epsilon + \epsilon \cdot \epsilon + O(\epsilon^2) \\
        = & O(\epsilon).
    \end{align*}
    More strictly, considering Assumption~\ref{assumption:dataset_adaptation} and the definition of the upper bound of the asymptotic growth rate, we have that the theorem holds.
\end{proof}

Theorem~\ref{thm:utility} shows that the model fine-tuned with the obfuscated dataset $\widetilde{\mathcal{D}}$ can also achieve a similar loss value on the original dataset $\mathcal{D}$.
This ensures that the data obfuscated by \codename has minimal impact on model utility if the obfuscated data is close enough to the original data in the embedding space.
This theorem can be symmetrically applied on $L(\bm{x}, \bm{\theta}^*)$ and $L(\tilde{\bm{x}}, \bm{\theta}^*)$ to show that the local stability on loss value of two embedding neighbors.

\subsubsection{Effective Defense against GIAs}

We aim to show the difference between the gradients produced by the original data point $\bm{x}$ and its obfuscated one $\tilde{\bm{x}}$ is significant enough to reduce the effectiveness of GIAs. The gradient difference has a linear asymptotic growth rate when $\epsilon \to 0$. This indicates that the gradient difference has a higher growth rate when $\epsilon$ approaches zero, compared to the loss value difference in Theorem~\ref{thm:utility}.

\begin{theorem}[Effective defense against GIAs]

\label{thm:effective_defense}
    Given the conditions in Theorem~\ref{thm:utility}, the gradient difference between gradients generated by the original data point $\bm{x} \in \mathcal{D}$ and its obfuscated data point $\tilde{\bm{x}} \in \widetilde{\mathcal{D}}$ satisfies the following lower asymptotic growth rate, when $\epsilon \to 0$,
    \begin{gather*}
        \| \nabla_{\tilde{\bm{\theta}}} L(\bm{x}; \tilde{\bm{\theta}}) - \nabla_{\tilde{\bm{\theta}}} L(\tilde{\bm{x}}; \tilde{\bm{\theta}}) \|
        = O(1),
    \end{gather*}
    where $\epsilon$ defined in Assumption~\ref{assumption:small_update}.
    More strictly, we have
    \begin{gather*}
        \limsup_{\epsilon \to 0}
        \| \nabla_{\tilde{\bm{\theta}}} L(\bm{x}; \tilde{\bm{\theta}}) - \nabla_{\tilde{\bm{\theta}}} L(\tilde{\bm{x}}; \tilde{\bm{\theta}}) \| < k,
    \end{gather*}
    where $k$ is a positive constant.
\end{theorem}
\begin{proof}

    Given the Taylor expansions of the loss function, which consider the first-order and second-order derivatives of the loss function and describe the model state from the pre-trained model to the model fine-tuned with the obfuscated dataset, we have
    \begin{gather*}
    \begin{aligned}
        L(\bm{x}; \bm{\theta})
        =~ &
        L(\bm{x}; \tilde{\bm{\theta}}) + \nabla_{\tilde{\bm{\theta}}} L(\bm{x}; \tilde{\bm{\theta}}) \Delta_{\tilde{\bm{\theta}}} + O(\| \Delta_{\tilde{\bm{\theta}}} \|^2), \\
        L(\tilde{\bm{x}}; \bm{\theta})
        =~ &
        L(\tilde{\bm{x}}; \tilde{\bm{\theta}}) + \nabla_{\tilde{\bm{\theta}}} L(\tilde{\bm{x}}; \tilde{\bm{\theta}}) \Delta_{\tilde{\bm{\theta}}} + O(\| \Delta_{\tilde{\bm{\theta}}} \|^2).
    \end{aligned}
    \end{gather*}
    Then, we have, by Assumption~\ref{assumption:dataset_modification} and Theorem~\ref{thm:utility},
    \begin{gather*}
    \begin{aligned}
        & \nabla_{\tilde{\bm{\theta}}} L(\bm{x}; \tilde{\bm{\theta}}) \Delta_{\tilde{\bm{\theta}}} - \nabla_{\tilde{\bm{\theta}}}  L(\tilde{\bm{x}}; \tilde{\bm{\theta}}) \Delta_{\tilde{\bm{\theta}}} \\
        \leq & | \nabla_{\tilde{\bm{\theta}}} L(\bm{x}; \tilde{\bm{\theta}}) \Delta_{\tilde{\bm{\theta}}} - \nabla_{\tilde{\bm{\theta}}}  L(\tilde{\bm{x}}; \tilde{\bm{\theta}}) \Delta_{\tilde{\bm{\theta}}} | \\
        = & | L(\bm{x}; \bm{\theta}) - L(\tilde{\bm{x}}; \bm{\theta}) + L(\tilde{\bm{x}}; \tilde{\bm{\theta}}) - L(\bm{x}; \tilde{\bm{\theta}}) + O(\| \Delta_{\tilde{\bm{\theta}}} \|^2) | \\
        = & | L(\bm{x}; \bm{\theta}) - L(\tilde{\bm{x}}; \bm{\theta}) | + | L(\tilde{\bm{x}}; \tilde{\bm{\theta}}) - L(\bm{x}; \tilde{\bm{\theta}}) | + | O(\| \Delta_{\tilde{\bm{\theta}}} \|^2) | \\
        \leq & \epsilon + O(\epsilon) + O(\epsilon^2) \\ 
        = & O(\epsilon).
    \end{aligned}
    \end{gather*}
    We prove this theorem by contradiction.
    We consider the asymptotic growth rate of each component of a vector in the following.
    Suppose one of $\nabla_{\tilde{\bm{\theta}}} L(\bm{x}; \tilde{\bm{\theta}})$ and $\nabla_{\tilde{\bm{\theta}}}  L(\tilde{\bm{x}}; \tilde{\bm{\theta}})$ satisfies a asymptotic growth rate higher than a constant one, \textit{e.g.}, a linear one.
    Since $\Delta_{\tilde{\theta}}$ has a linear asymptotic growth rate, as $0 \leq \| \Delta_{\tilde{\theta}} \| \leq \epsilon = O(\epsilon)$, then $\nabla_{\tilde{\bm{\theta}}} L(\bm{x}; \tilde{\bm{\theta}}) \Delta_{\tilde{\bm{\theta}}} - \nabla_{\tilde{\bm{\theta}}}  L(\tilde{\bm{x}}; \tilde{\bm{\theta}}) \Delta_{\tilde{\bm{\theta}}}$ has an asymptotic growth rate higher than $O(\epsilon)$, which is a contradiction to the above result.
    So, $\nabla_{\tilde{\bm{\theta}}} L(\bm{x}; \tilde{\bm{\theta}})$ and $\nabla_{\tilde{\bm{\theta}}} L(\tilde{\bm{x}}; \tilde{\bm{\theta}})$ have $O(1)$ as their asymptotic growth rate and $\| \nabla_{\tilde{\bm{\theta}}} L(\bm{x}; \tilde{\bm{\theta}})  - \nabla_{\tilde{\bm{\theta}}} L(\tilde{\bm{x}}; \tilde{\bm{\theta}}) \| = O(1)$.
    Moreover, by the definition of the upper bound of the asymptotic growth rate, we have that the theorem holds.
\end{proof}

Note that, when considering $\epsilon \to 0$, the constant asymptotic growth rate $O(1)$ is more significant compared to the linear asymptotic growth rate $O(\epsilon)$.
Moreover, note that Theorem~\ref{thm:utility} and Theorem~\ref{thm:effective_defense} are considering the same order of Taylor expansion.
It shows that the loss value difference always has a higher order of asymptotic growth rate than the gradient difference.
This indicates a similar model performance on the original dataset and the obfuscated dataset, but the gradient has a potentially more significant difference.
This theorem can also be symmetrically applied on $\nabla_{\bm{\theta}}^* L(\bm{x}, \bm{\theta}^*)$ and $\nabla_{\bm{\theta}}^* L(\tilde{\bm{x}}, \bm{\theta}^*)$ to show the gradient difference of two neighbor inputs are less stable than their loss value difference. In Section~\ref{exp: validation_theory}, we experimentally validate that the growth rate of the gradient is larger than the one of the loss when fine-tuning on the obfuscated data.


\section{Experimental Evaluation}

This section outlines the evaluation of \codename.

\subsection{Experimental Settings}
\label{sec:experimental_settings}
We introduce the experimental settings, covering language models, tasks and datasets, FL setups, and fine-tuning settings. A hyperparameters study can be found in Appendix~\ref{appendix:hyper}.

\myparagraph{Language models}
\codename is evaluated on both MLMs and CLMs, spanning six model families with 21 models. The MLMs include BERT~\cite{bert} (cased, uncased, base, large), RoBERTa~\cite{roberta} (base, large), and DeBERTa-v3~\cite{he2020deberta} (small, base, large). The CLMs consist of GPT-2~\cite{gpt2} (base, medium, large, extra-large), Llama~\cite{touvron2023llama} (2-7b, 2-13b, 3-8b, 3.1-8b, 3.2-1b, 3.2-3b), and Gemma~\cite{team2024gemma} (2-2b, 2-9b). 
This diverse range of models provides a comprehensive evaluation of \codename across different architectures and applications.

\myparagraph{Tasks and datasets} 
We evaluate \codename with widely used datasets across sentence classification and sentence generation tasks, which are the main targeted tasks in GIAs.

For the sentence classification task, we use the SST-2~\cite{sst2_dataset} (binary), Tweet Sentiment Analysis~\cite{tweet_sentiment_dataset} (3 classes), and the Yahoo Answers Topics~\cite{yahoo_answers_dataset} (10 classes) datasets. We randomly sample 1,000, 3,000, and 10,000 samples from these datasets, respectively.
The samples are randomly selected and balanced across classes to ensure fair evaluation.

For the sentence generation task, we use the
Enron Emails\footnote{We use the preprocessed version of the dataset provided by the authors of FILM~\cite{film}, where the dataset is split into sentences.}~\cite{enron_emails_dataset}, Open Australian Legal Corpus~\cite{open_australian_legal_corpus}, and AG News~\cite{ag_news_dataset}. To align with the scope of current attacks, only sentences with fewer than 30 words are included. We randomly select 1,000 data samples for each dataset. 

All datasets follow a standard 0.8-0.2 train-test split during training and evaluation.

\myparagraph{FL setups} In this work, we consider the FedSGD~\cite{fl, ma2023loden, zhang2023agrevader, ma2024unveiling, vo2025practical} setting, which involves sharing gradients after each local training iteration. A commonly used setting in existing attacks~\cite{dlg, tag, lamp, film, grab, dager} assumes one server and one client (victim) for simplicity, as success on one implies vulnerability of all. This setup is conceptually equivalent to scenarios involving multiple clients in FL~\cite{film}.
To align with existing attacks, we also adopt this setting, where the victim client goes through a fine-tuning with all obfuscated data of a dataset (equivalent to splitting the obfuscated data with multiple clients), and the resulting gradients are being observed by the attacker.

\myparagraph{Fine-tuning settings} 
For all MLMs, the appended classification layer is initialized randomly, and the entire model undergoes full-parameter fine-tuning. In the case of CLMs, the GPT-2 series is fine-tuned with full-parameter optimization. For the Llama and Gemma series, Low-Rank Adaptation (LoRA)~\cite{hu2021lora} is employed with default settings (rank 8, alpha 16, dropout 0.1, and no bias) and 4-bit quantization, where only the LoRA blocks are fine-tuned. All datasets are split into training and testing sets with a 0.8--0.2 random split. Fine-tuning is conducted with a batch size of 32 and a learning rate of 1e-5. Early stopping is used, with a patience of 5 epochs.

\begin{table*}[!htbp]

\renewcommand{\arraystretch}{.8}
\setlength{\tabcolsep}{8pt}

\centering
\caption{Similarity between original and obfuscated data.}
\label{table:privacy_preservation}

\begin{threeparttable}
\begin{tabular}{
    l
    rrrr p{0pt}
    rrrr p{0pt}
    rrrr p{0pt}
}

\toprule
& 
\multicolumn{4}{c}{SST-2} && 
\multicolumn{4}{c}{Tweet Sentiment Analysis} && 
\multicolumn{4}{c}{Yahoo Answers Topics} \\
\cmidrule(l{1pt}r{1pt}){2-5} 
\cmidrule(l{1pt}r{1pt}){7-10}
\cmidrule(l{1pt}r{1pt}){12-15}
& 
\makecell[c]{R-1} & \makecell[c]{R-2} & \makecell[c]{R-L} & \makecell[c]{M} &&
\makecell[c]{R-1} & \makecell[c]{R-2} & \makecell[c]{R-L} & \makecell[c]{M} &&
\makecell[c]{R-1} & \makecell[c]{R-2} & \makecell[c]{R-L} & \makecell[c]{M} \\

\midrule

BERT-base-uncased & 
0.02 & 0.00 & 0.02 & 0.03 &&
0.02 & 0.00 & 0.02 & 0.03 &&
0.02 & 0.00 & 0.02 & 0.03 \\

BERT-base-cased & 
0.02 & 0.00 & 0.02 & 0.02 && 
0.02 & 0.00 & 0.02 & 0.03 && 
0.02 & 0.00 & 0.02 & 0.02 \\

BERT-large-uncased & 
0.03 & 0.00 & 0.03 & 0.03 && 
0.03 & 0.00 & 0.03 & 0.03 && 
0.03 & 0.00 & 0.02 & 0.03 \\

BERT-large-cased & 
0.03 & 0.00 & 0.02 & 0.03 && 
0.03 & 0.00 & 0.03 & 0.04 && 
0.03 & 0.00 & 0.03 & 0.03 \\

RoBERTa-base & 
0.01 & 0.00 & 0.01 & 0.01 && 
0.02 & 0.00 & 0.02 & 0.03 && 
0.01 & 0.00 & 0.01 & 0.02 \\

RoBERTa-large & 
0.04 & 0.00 & 0.03 & 0.03 && 
0.03 & 0.00 & 0.02 & 0.03 && 
0.03 & 0.00 & 0.03 & 0.03 \\

DeBERTa-v3-small & 
0.01 & 0.00 & 0.01 & 0.01 && 
0.02 & 0.00 & 0.02 & 0.02 && 
0.02 & 0.00 & 0.02 & 0.02 \\

DeBERTa-v3-base & 
0.01 & 0.00 & 0.01 & 0.01 && 
0.03 & 0.00 & 0.03 & 0.03 && 
0.02 & 0.00 & 0.02 & 0.02 \\

DeBERTa-v3-large & 
0.02 & 0.00 & 0.02 & 0.03 && 
0.03 & 0.00 & 0.03 & 0.03 && 
0.03 & 0.00 & 0.03 & 0.03 \\
\midrule

& 
\multicolumn{4}{c}{Enron Emails} && 
\multicolumn{4}{c}{Open Australian Legal Corpus} && 
\multicolumn{4}{c}{AG News} \\
\cmidrule(l{1pt}r{1pt}){2-5} 
\cmidrule(l{1pt}r{1pt}){7-10}
\cmidrule(l{1pt}r{1pt}){12-15}

& 
\makecell[c]{R-1} & \makecell[c]{R-2} & \makecell[c]{R-L} & \makecell[c]{M} &&
\makecell[c]{R-1} & \makecell[c]{R-2} & \makecell[c]{R-L} & \makecell[c]{M} &&
\makecell[c]{R-1} & \makecell[c]{R-2} & \makecell[c]{R-L} & \makecell[c]{M} \\

\midrule

GPT-2 & 
0.03 & 0.00 & 0.03 & 0.04 &&
0.05 & 0.00 & 0.04 & 0.05 &&
0.04 & 0.00 & 0.03 & 0.05 \\

GPT-2-medium & 
0.03 & 0.00 & 0.03 & 0.05 &&
0.04 & 0.00 & 0.04 & 0.06 &&
0.04 & 0.00 & 0.03 & 0.06 \\

GPT-2-large & 
0.03 & 0.00 & 0.02 & 0.04 &&
0.03 & 0.00 & 0.03 & 0.04 &&
0.03 & 0.00 & 0.03 & 0.04 \\

GPT-2-xl & 
0.03 & 0.00 & 0.03 & 0.04 &&
0.05 & 0.00 & 0.04 & 0.05 &&
0.03 & 0.00 & 0.03 & 0.04 \\

Llama-2-7b & 
0.02 & 0.00 & 0.02 & 0.04 &&
0.02 & 0.00 & 0.02 & 0.04 &&
0.02 & 0.00 & 0.02 & 0.04 \\

Llama-2-13b & 
0.02 & 0.00 & 0.02 & 0.04 &&
0.04 & 0.00 & 0.04 & 0.05 &&
0.03 & 0.00 & 0.03 & 0.04 \\

Llama-3-8b & 
0.08 & 0.00 & 0.07 & 0.07 &&
0.09 & 0.00 & 0.09 & 0.08 &&
0.07 & 0.00 & 0.06 & 0.06 \\

Llama-3.1-8b & 
0.08 & 0.00 & 0.08 & 0.07 &&
0.09 & 0.00 & 0.08 & 0.08 &&
0.07 & 0.00 & 0.06 & 0.06 \\

Llama-3.2-1b & 
0.08 & 0.00 & 0.08 & 0.07 &&
0.12 & 0.01 & 0.11 & 0.07 &&
0.08 & 0.00 & 0.07 & 0.05 \\

Llama-3.2-3b & 
0.09 & 0.00 & 0.09 & 0.08 &&
0.11 & 0.00 & 0.10 & 0.07 &&
0.07 & 0.00 & 0.06 & 0.05 \\

Gemma-2-2b & 
0.06 & 0.00 & 0.05 & 0.06 &&
0.09 & 0.00 & 0.07 & 0.07 &&
0.06 & 0.00 & 0.05 & 0.05 \\

Gemma-2-9b & 
0.05 & 0.00 & 0.04 & 0.06 &&
0.08 & 0.00 & 0.06 & 0.07 &&
0.06 & 0.00 & 0.05 & 0.06 \\
\bottomrule
\end{tabular}

\end{threeparttable}
\end{table*}

\subsection{Validation of Theoretical Analysis}
\label{exp: validation_theory}
We begin by validating our theoretical analysis. Specifically, we aim to show the existence of the joint upper bound $\epsilon$ of the small positive values in Assumption~\ref{assumption:small_update} and \ref{assumption:dataset_modification}, and the growth rate of gradient is larger than the one of the loss in the Theorems. We use the BERT-base-uncased models fine-tuned on the original and the obfuscated SST-2 training splits (800 samples each) for demonstration.
\myparagraph{Joint upper bound} We compute the values of the terms in the assumptions, including the parameter magnitude drift for model fine-tuned on original data, \ie, $\epsilon_{\bm{\theta}^*}$ (1.44), the parameter magnitude drift for model fine-tuned on obfuscated data, \ie, $\epsilon_{\tilde{\bm{\theta}}}$ (2.48), the loss deviation between the original data and the obfuscated data of the pre-trained model, \ie, $\epsilon_{L_{\bm{\theta}}}$ (0.04 mean), and the gradient deviation between the original data and the obfuscated data of the pre-trained model, \ie, $\epsilon_{g_{\bm{\theta}}}$ (8.45 mean).  For each original and obfuscated sample pair, we compute a global upper bound $\epsilon = \text{max}(\epsilon_{\bm{\theta}^*}, \epsilon_{\tilde{\bm{\theta}}}, \epsilon_{L_{\bm{\theta}}}, \epsilon_{g_{\bm{\theta}}})$, and the local upper bound $\epsilon = \text{max}(\epsilon_{L_{\bm{\theta}}}, \epsilon_{g_{\bm{\theta}}})$ for the deviation terms in Assumption \ref{assumption:dataset_modification}. Figure~\ref{fig:epsilons} in Appendix~\ref{appendix: validation_theory} shows their correlations. Subplot (1,1) indicates that $\epsilon$ is concentrated in the lower range, implying small parameter drifts and deviations for most samples. Subplot (2,2) shows the same pattern for $\epsilon_d$, suggesting it dominates the global bound. The scatter plots in subplots (1,2) and (2,1) further confirm this, where values lie along a diagonal with a slope of 1.

\myparagraph{Growth rate validation} For each sample pair, we compute the loss deviations (0.47 mean) and the gradient deviations (40.21 mean) of the obfuscated model between the original data and the obfuscated data. The results are visualized in Figure~\ref{fig:loss_vs_grad} in Appendix~\ref{appendix: validation_theory}, where both values are plotted in a pair-wise scatter plot, and a linear regression is applied to show the correlation between $\epsilon$ and the deviations. The vertical axis uses a logarithmic scale for clearer separation of points, although the regression is computed in linear scale, making the fitted lines appear relatively flat. Loss deviations remain consistently small, while gradient deviations are substantially larger. Crucially, the fitted slope for gradient deviations far exceeds that for loss deviations, confirming the theoretical insight that utility is preserved (slow loss-deviation growth) while inversion is defended (rapid gradient-deviation growth).

\subsection{Data Obfuscation} 
\label{sec:data_obfuscation}
We evaluate \codename's data obfuscation capability by comparing the similarity between the obfuscated and original data across various models and datasets.

\myparagraph{Metrics}
Both exact match and partial semantic match similarities are measured. Exact match similarity is evaluated with ROUGE score~\cite{rouge}. ROUGE considers overlapping n-grams (unigrams, bigrams) and the longest matching subsequence, reported as ROUGE-1 (R-1), ROUGE-2 (R-2), and ROUGE-L (R-L). Partial semantic similarity is assessed with METEOR (M) score~\cite{banerjee2005meteor}, which captures a wider range of matches, including exact words, stems, synonyms, and paraphrases. Both metrics range from 0 to 1, where higher values indicate greater similarity, hence lower obfuscation level.

\myparagraph{Results}
The results in Table~\ref{table:privacy_preservation} demonstrate that, for most models and datasets, the obfuscated data shows negligible similarity to the original data, as reflected by ROUGE and METEOR scores. R-1 scores range from 0.01 to 0.12, while METEOR scores fall between 0.01 and 0.08, indicating minimum overlap in both individual tokens and partial semantics. This confirms \codename's remarkable effectiveness in obfuscating the original data.

An original-obfuscated sentence pair generated by \codename is shown below, where the obfuscated sentence contains limited semantics similarity compared to the original sentence (\textit{e.g.}, ends--concluded, six--THREE, stalled--discouraged). A human evaluation on the semantics difference between the original-obfuscated pairs is presented in Appendix~\ref{appendix:human_eval_semantics}.
\definecolor{skyblue}{rgb}{0.529, 0.808, 0.922}
\begin{tcolorbox}[
    fontupper=\footnotesize,
    colback=white,    
    colframe=skyblue!40!white, 
    width=\columnwidth,       
    sharp corners,           
    title=Example of original and obfuscated sentence pair., 
    fonttitle=\bfseries,     
    coltitle=black,           
    left=0mm, right=0mm, top=1mm, bottom=1mm
]
\textbf{Original sentence:} \\
The compromise apparently ends six months of stalled negotiations.

\textbf{Obfuscated sentence:} \\
Becauseromising bizarre concluded THREE kilometers had discouraged tensions.
\end{tcolorbox}

We emphasize that the other unassessed semantic similarity metrics, such as BERTScore~\cite{zhang2019bertscore}, rely on embedding similarity, which directly conflicts with our approach. 

\subsection{Utility Preservation}
\label{sec:utility_preservation}

\begin{table*}[!htbp]

\renewcommand{\arraystretch}{.8}
\setlength{\tabcolsep}{2pt}

\centering
\caption{Utility comparison across three model configurations.}
\label{table:utility_preservation}

\begin{threeparttable}
\resizebox{1\textwidth}{!}{
\begin{tabular}{
    l
    rrrrrr p{0pt}
    rrrrrr p{0pt}
    rrrrrr p{0pt}
}

\toprule
& \multicolumn{6}{c}{SST-2} && 
\multicolumn{6}{c}{Tweet Sentiment Analysis} && 
\multicolumn{6}{c}{Yahoo Answers Topics} \\
\cmidrule(l{1pt}r{1pt}){2-7} 
\cmidrule(l{1pt}r{1pt}){9-14}
\cmidrule(l{1pt}r{1pt}){16-21}
& \multicolumn{2}{c}{Pre-trained} & \multicolumn{2}{c}{Original} & \multicolumn{2}{c}{Obfuscated} && 
\multicolumn{2}{c}{Pre-trained} & \multicolumn{2}{c}{Original} & \multicolumn{2}{c}{Obfuscated} &&
\multicolumn{2}{c}{Pre-trained} & \multicolumn{2}{c}{Original} & \multicolumn{2}{c}{Obfuscated} \\
\cmidrule(l{1pt}r{1pt}){2-3} 
\cmidrule(l{1pt}r{1pt}){4-5} 
\cmidrule(l{1pt}r{1pt}){6-7}
\cmidrule(l{1pt}r{1pt}){9-10} 
\cmidrule(l{1pt}r{1pt}){11-12}
\cmidrule(l{1pt}r{1pt}){13-14}
\cmidrule(l{1pt}r{1pt}){16-17} 
\cmidrule(l{1pt}r{1pt}){18-19}
\cmidrule(l{1pt}r{1pt}){20-21}
& 
\multicolumn{1}{c}{Acc.} & \multicolumn{1}{c}{F1} & \multicolumn{1}{c}{Acc.} & \multicolumn{1}{c}{F1} & \multicolumn{1}{c}{Acc.} & \multicolumn{1}{c}{F1} &&
\multicolumn{1}{c}{Acc.} & \multicolumn{1}{c}{F1} & \multicolumn{1}{c}{Acc.} & \multicolumn{1}{c}{F1} & \multicolumn{1}{c}{Acc.} & \multicolumn{1}{c}{F1} &&
\multicolumn{1}{c}{Acc.} & \multicolumn{1}{c}{F1} & \multicolumn{1}{c}{Acc.} & \multicolumn{1}{c}{F1} & \multicolumn{1}{c}{Acc.} & \multicolumn{1}{c}{F1} \\
\midrule
BERT-base-uncased & 
0.53 & 0.64 & 0.92 & 0.93 & 0.82 & 0.85 &&
0.27 & 0.15 & 0.76 & 0.75 & 0.63 & 0.63 &&
0.11 & 0.03 & 0.65 & 0.65 & 0.60 & 0.60 \\

BERT-base-cased & 
0.54 & 0.70 & 0.89 & 0.90 & 0.84 & 0.85 &&
0.28 & 0.15 & 0.73 & 0.73 & 0.58 & 0.57 &&
0.10 & 0.03 & 0.64 & 0.64 & 0.58 & 0.58 \\

BERT-large-uncased & 
0.54 & 0.70 & 0.94 & 0.94 & 0.79 & 0.83 &&
0.33 & 0.27 & 0.74 & 0.73 & 0.61 & 0.59 &&
0.12 & 0.04 & 0.66 & 0.66 & 0.47 & 0.48 \\

BERT-large-cased & 
0.48 & 0.04 & 0.91 & 0.91 & 0.81 & 0.83 &&
0.36 & 0.19 & 0.75 & 0.73 & 0.63 & 0.60 &&
0.09 & 0.03 & 0.64 & 0.64 & 0.59 & 0.59 \\

RoBERTa-base & 
0.47 & 0.00    & 0.91 & 0.91 & 0.77 & 0.81 &&
0.35 & 0.17 & 0.79 & 0.79 & 0.61 & 0.60 &&
0.10 & 0.02 & 0.65 & 0.65 & 0.60 & 0.59 \\

RoBERTa-large & 
0.47 & 0.00    & 0.93 & 0.93 & 0.78 & 0.81 &&
0.35 & 0.17 & 0.79 & 0.77 & 0.60 & 0.58 &&
0.10 & 0.02 & 0.68 & 0.67 & 0.64 & 0.63 \\

DeBERTa-v3-small & 
0.47 & 0.00    & 0.94 & 0.94 & 0.84 & 0.86 &&
0.36 & 0.18 & 0.78 & 0.77 & 0.63 & 0.59 &&
0.10 & 0.02 & 0.64 & 0.64 & 0.49 & 0.48 \\

DeBERTa-v3-base & 
0.54 & 0.70 & 0.94 & 0.94 & 0.69 & 0.76 &&
0.29 & 0.15 & 0.77 & 0.75 & 0.61 & 0.53 &&
0.11 & 0.02 & 0.65 & 0.65 & 0.45 & 0.43 \\

DeBERTa-v3-large & 
0.52 & 0.46 & 0.94 & 0.94 & 0.92 & 0.92 &&
0.35 & 0.17 & 0.77 & 0.76 & 0.60 & 0.51 &&
0.10 & 0.02 & 0.67 & 0.67 & 0.63 & 0.63 \\
\midrule
& \multicolumn{6}{c}{Enron Emails} && 
\multicolumn{6}{c}{Open Australian Legal Corpus} && 
\multicolumn{6}{c}{AG News} \\
\cmidrule(l{1pt}r{1pt}){2-7} 
\cmidrule(l{1pt}r{1pt}){9-14}
\cmidrule(l{1pt}r{1pt}){16-21}
& \multicolumn{2}{c}{Pre-trained} & \multicolumn{2}{c}{Original} & \multicolumn{2}{c}{Obfuscated} && 
\multicolumn{2}{c}{Pre-trained} & \multicolumn{2}{c}{Original} & \multicolumn{2}{c}{Obfuscated} &&
\multicolumn{2}{c}{Pre-trained} & \multicolumn{2}{c}{Original} & \multicolumn{2}{c}{Obfuscated} \\
\cmidrule(l{1pt}r{1pt}){2-3} 
\cmidrule(l{1pt}r{1pt}){4-5} 
\cmidrule(l{1pt}r{1pt}){6-7}
\cmidrule(l{1pt}r{1pt}){9-10} 
\cmidrule(l{1pt}r{1pt}){11-12}
\cmidrule(l{1pt}r{1pt}){13-14}
\cmidrule(l{1pt}r{1pt}){16-17} 
\cmidrule(l{1pt}r{1pt}){18-19}
\cmidrule(l{1pt}r{1pt}){20-21}
& 
\multicolumn{1}{c}{Loss} & \multicolumn{1}{c}{PPL} & \multicolumn{1}{c}{Loss} & \multicolumn{1}{c}{PPL} & \multicolumn{1}{c}{Loss} & \multicolumn{1}{c}{PPL} &&
\multicolumn{1}{c}{Loss} & \multicolumn{1}{c}{PPL} & \multicolumn{1}{c}{Loss} & \multicolumn{1}{c}{PPL} & \multicolumn{1}{c}{Loss} & \multicolumn{1}{c}{PPL} &&
\multicolumn{1}{c}{Loss} & \multicolumn{1}{c}{PPL} & \multicolumn{1}{c}{Loss} & \multicolumn{1}{c}{PPL} & \multicolumn{1}{c}{Loss} & \multicolumn{1}{c}{PPL} \\
\midrule

GPT-2 & 
8.17 & 3540.01 & 2.40 & 11.05 & 2.95 & 19.12 &&
8.86 & 7053.71 & 1.78 &  5.91 & 2.01 &  7.46 &&
8.18 & 3578.13 & 2.20 &  9.07 & 2.56 & 12.91 \\

GPT-2-medium & 
7.66 & 2122.13 & 2.29 & 9.87 & 2.83 & 16.92 &&
8.30 & 4032.74 & 1.65 & 5.22 & 1.94 &  6.94 &&
7.48 & 1777.53 & 2.02 & 7.51 & 2.36 & 10.61 \\

GPT-2-large & 
8.01 & 3001.36 & 2.23 & 9.30 & 2.91 & 16.69 &&
8.75 & 6304.60 & 1.57 & 4.79 & 1.83 &  6.24 &&
7.88 & 2650.71 & 1.92 & 6.82 & 2.34 & 10.42 \\

GPT-2-xl & 
7.55 & 1902.85 & 2.21 & 9.08 & 3.50 & 33.22 &&
8.14 & 3421.60 & 1.54 & 4.65 & 2.65 & 14.16 &&
7.36 & 1565.56 & 1.88 & 6.58 & 2.81 & 16.64 \\

Llama-2-7b & 
15.81 & 7.36e06 & 1.23 & 3.42 & 1.73 & 5.64 &&
14.96 & 3.13e06 & 1.32 & 3.76 & 1.70 & 5.45 &&
11.08 & 6.45e04 & 1.57 & 4.79 & 1.94 & 6.95 \\

Llama-2-13b & 
15.90 & 8.02e06 & 1.25 & 3.48 & 2.73 & 15.27 &&
13.55 & 7.65e05 & 1.34 & 3.84 & 3.27 & 26.29 &&
10.25 & 2.82e04 & 1.19 & 3.30 & 4.07 & 58.65 \\

Llama-3-8b$^*$ & 
5.57 & 263.44 & 1.86 & 6.45 & 2.40 & 11.03 &&
4.85 & 127.88 & 1.74 & 5.72 & 2.09 & 8.12 &&
4.49 & 72.81 & 1.75 & 5.74 & 2.04 & 7.65 \\

Llama-3.1-8b$^*$ & 
6.66 & 778.83 & 1.88 & 6.53 & 2.39 & 10.89 &&
5.50 & 244.82 & 1.75 & 5.76 & 2.09 &  8.08 &&
5.45 & 233.92 & 1.75 & 5.75 & 2.05 &  7.75 \\

Llama-3.2-1b$^*$ & 
7.47 & 1756.23 & 2.03 & 7.58 & 2.50 & 12.21 &&
6.69 &  801.26 & 1.93 & 6.88 & 2.23 & 9.29 &&
6.10 &  447.35 & 2.01 & 7.45 & 2.28 & 9.76 \\

Llama-3.2-3b$^*$ & 
8.16 & 3511.35 & 1.93 & 6.92 & 2.37 & 10.72 &&
7.11 & 1226.99 & 1.83 & 6.23 & 2.15 &  8.62 &&
6.94 & 1030.13 & 1.84 & 6.29 & 2.14 &  8.53 \\

Gemma-2-2b & 
22.78 & 7.81e09 & 1.56 & 4.77 & 2.18 & 8.86 &&
18.75 & 1.38e08 & 1.75 & 5.76 & 2.49 & 12.03 &&
15.32 & 4.50e06 & 2.04 & 7.73 & 2.99 & 19.87 \\

Gemma-2-9b & 
27.15 & 6.19e11 & 1.24 & 3.47 & 1.78 & 5.93 &&
22.35 & 5.09e09 & 1.47 & 4.35 & 1.97 & 7.20 &&
17.35 & 3.43e07 & 1.72 & 5.57 & 2.99 & 11.88 \\
\bottomrule
\end{tabular}
}

\begin{tablenotes}
    \footnotesize
    \item[$*$] Datasets might be used within the pre-training, as indicated by lower loss and perplexity on pre-trained models. However, fine-tuning with obfuscated data is still effective in improving the utility.
\end{tablenotes}

\end{threeparttable}
\end{table*}

We assess the utility preservation capability of \codename by analyzing its impact on model performance across various models and datasets. 
The assessment is conducted on three model configurations across all models: 
(1) pre-trained models, 
(2) models fine-tuned on the original data, and 
(3) models fine-tuned on the obfuscated data.
The model performance is measured using the test split of the original data.

\subsubsection{Classification models} We present the utility evaluation on classification models.

\myparagraph{Metrics}
Accuracy (Acc.) and F1 scores are measured. Both metrics range from 0 to 1, where a higher value represents better utility performance.
\myparagraph{Results}
The results in Table~\ref{table:utility_preservation} show that most pre-trained classification models perform close to random guessing. Models fine-tuned on original data achieve high utility, with accuracies and F1 scores both ranging from 0.64 to 0.94 across different classification tasks. Models fine-tuned on obfuscated data exhibit similar, although slightly reduced, utility, with accuracies ranging from 0.45 to 0.92 and F1 scores from 0.43 to 0.92. This slight reduction illustrates the trade-off introduced by the obfuscation process while demonstrating that the obfuscated data preserves the critical features of the original data required for effective fine-tuning in classification tasks.

\subsubsection{Generative models}
\label{gen_model_utility}
We summarize the utility evaluation on generative models.
\myparagraph{Metrics} Loss and perplexity (PPL) are measured. Both metrics are direct indicators of model performance with no predetermined ranges, which are commonly employed utility metrics in sentence generation tasks~\cite{touvron2023llama, gpt2} and in GIAs targeting generative language models~\cite{film}. Perplexity is the exponential of the loss in generative language models. Lower non-negative values indicate better performance.

\myparagraph{Results}
The results in Table~\ref{table:utility_preservation} reveal a similar trend to classification models. Pre-trained models exhibit high losses and perplexities, reflecting lower performance. Fine-tuning with original data significantly improves these metrics, with losses ranging from 1.19 to 2.40 and perplexities from 3.30 to 11.05. Models fine-tuned on obfuscated data show slightly higher losses and perplexities, ranging from 1.70 to 4.07 and 5.45 to 58.65, respectively. While obfuscation introduces a minor degradation, the results confirm that the critical features for generative tasks are preserved. A human evaluation on the generation similarity between the original and obfuscated models is presented in Appendix~\ref{appendix:human_eval_generation}.

\subsection{Defense Efficacy}
\label{sec:defense_efficacy}

We evaluate the defense efficacy of \codename against GIAs using five existing passive GIA methods to align with our threat model of an honest-but-curious attacker:
\begin{itemize}[left=0pt]
    \item \textbf{DLG}~\cite{dlg}: The first attempt to reconstruct training data from gradients in language models.
    \item \textbf{TAG}~\cite{tag}: An enhanced method introducing additional regularization for improved token recovery.
    \item \textbf{LAMP}~\cite{lamp}: An attack leveraging discrete optimization techniques to achieve more precise reconstructions.
    \item \textbf{GRAB}~\cite{grab}: The first approach to consider practical training scenarios with hybrid optimization strategies.
    \item \textbf{DAGER}~\cite{dager}: A cutting-edge attack exploiting the rank of attention matrices for near-exact data recovery.
\end{itemize}

\begin{figure}[!htbp]

    \centering
    \begin{subfigure}[b]{0.49\linewidth}
        \includegraphics[width=\linewidth]{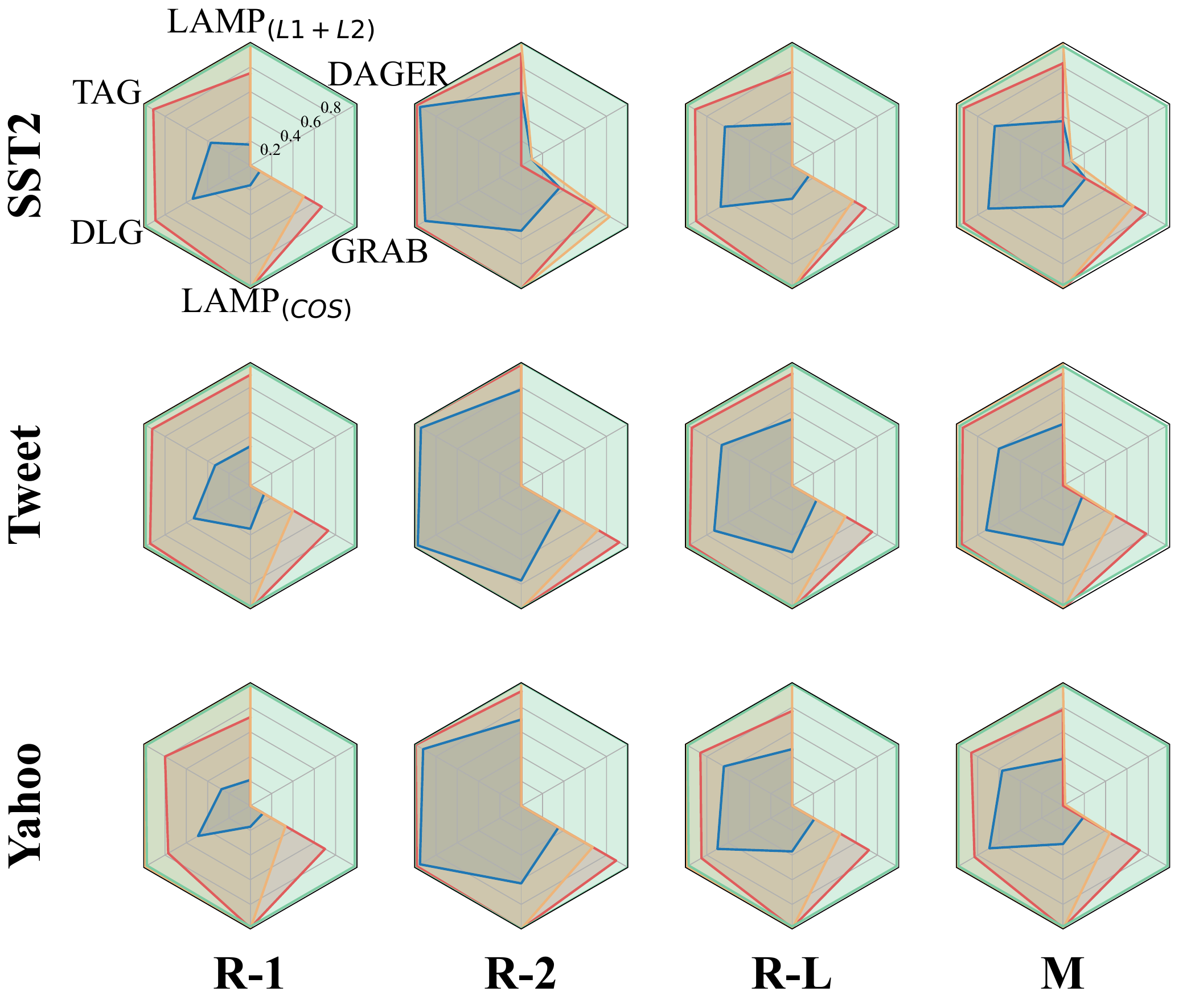}
        \caption{BERT-base-uncased}
    \end{subfigure}
    \hfill
    \begin{subfigure}[b]{0.49\linewidth}
        \includegraphics[width=\linewidth]{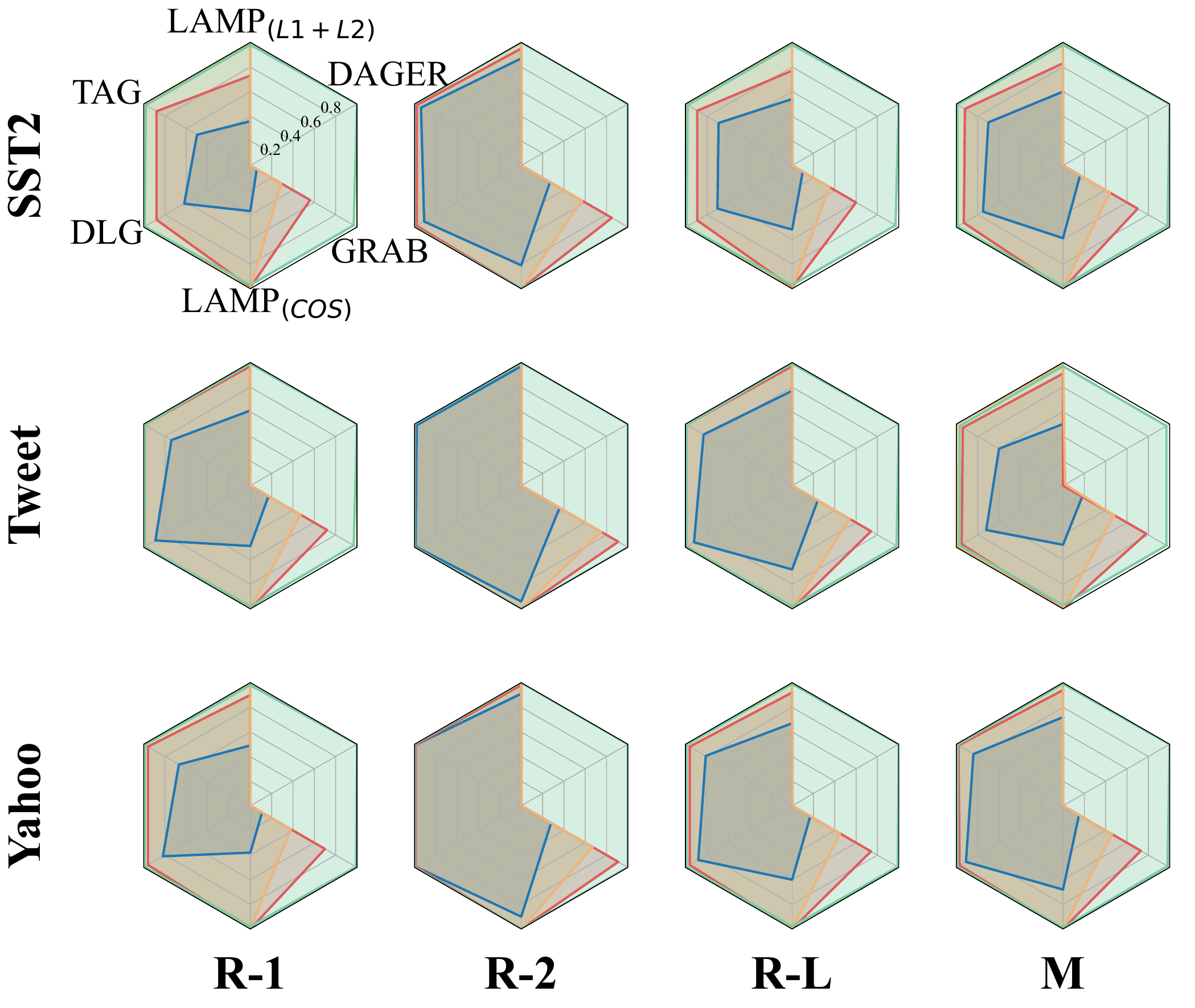}
        \caption{BERT-base-cased}
    \end{subfigure}
        \begin{subfigure}[b]{0.49\linewidth}
        \includegraphics[width=\linewidth]{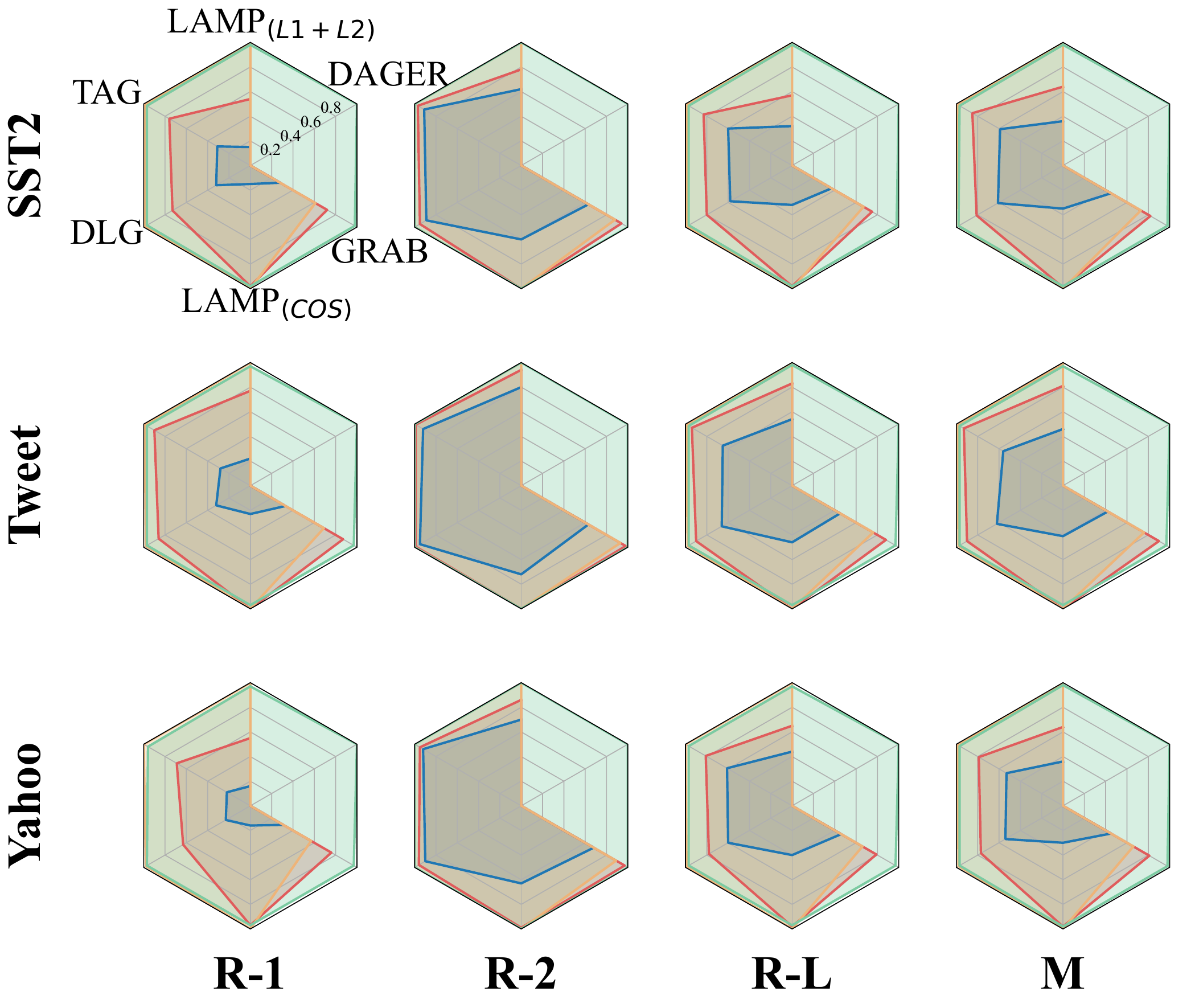}
        \caption{BERT-large-uncased}
    \end{subfigure}
    \hfill
    \begin{subfigure}[b]{0.49\linewidth}
        \includegraphics[width=\linewidth]{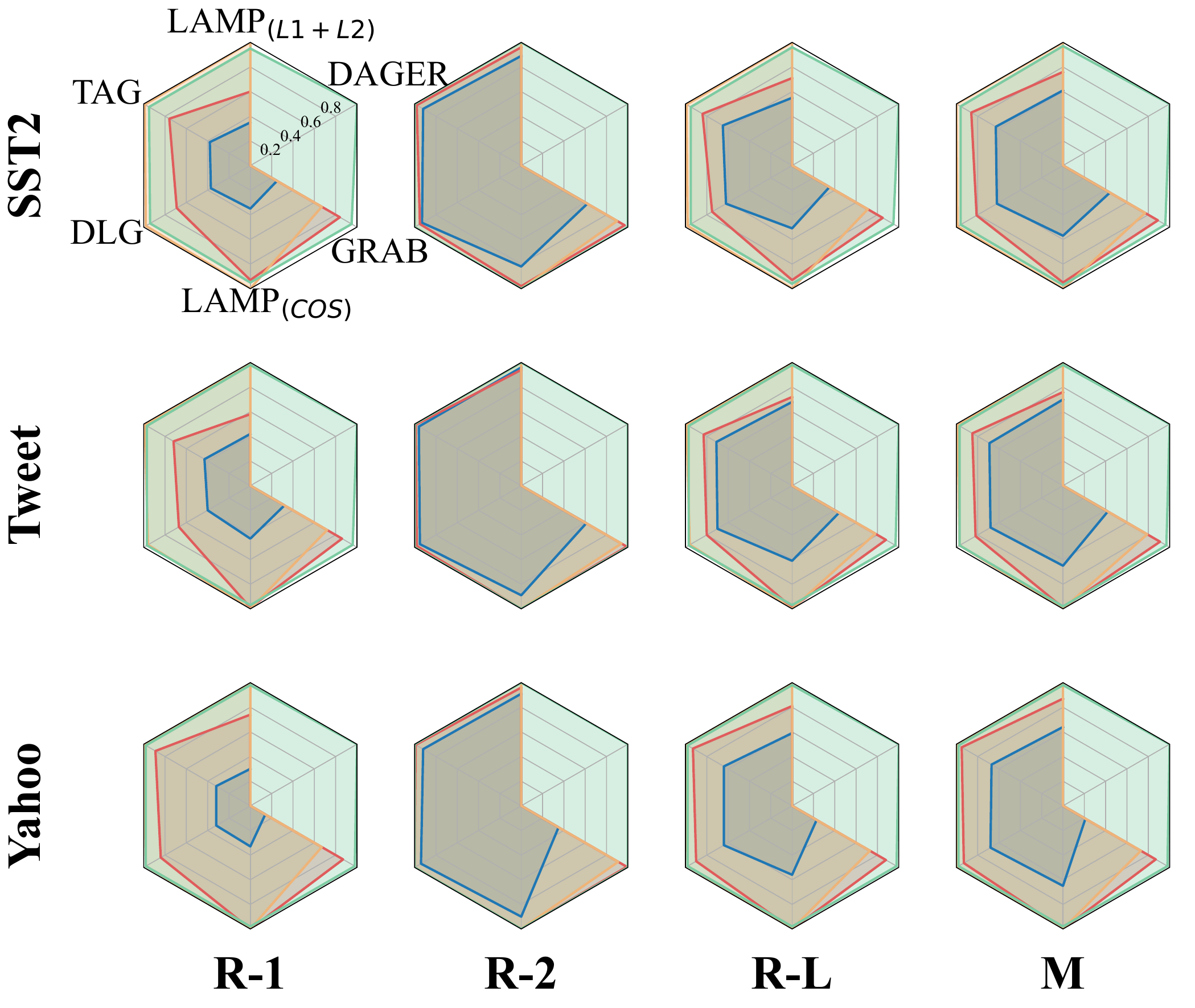}
        \caption{BERT-large-cased}
    \end{subfigure}
    
    \vspace{3pt}
    
    \begin{subfigure}[b]{\linewidth}
    \centering
        \includegraphics[width=0.8\linewidth]{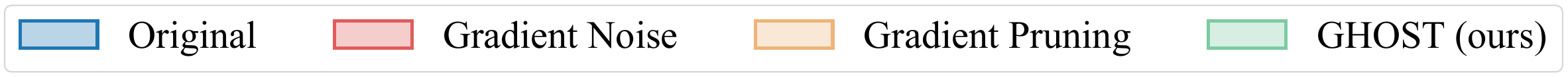}
    \end{subfigure}
    \caption{Defense efficacy against GIAs on BERT models (the greater coverage, the stronger defense).}
    \label{fig:classification_gia}
\end{figure}
\begin{table}[!htp]

\renewcommand{\arraystretch}{.8}
\setlength{\tabcolsep}{1pt}

\centering
\caption{GIA attack efficacy on generative models for DAGER.}
\label{table:generative_gia}

\begin{threeparttable}

\resizebox{\linewidth}{!}{
\begin{tabular}{
l @{\hspace{3pt}}
rrrr p{0pt}
rrrr p{0pt}
rrrr p{0pt}
}
\toprule
& \multicolumn{4}{c}{Original} && \multicolumn{4}{c}{Pruning} && \multicolumn{4}{c}{\codename} \\
\cmidrule(l{1pt}r{1pt}){2-5} 
\cmidrule(l{1pt}r{1pt}){7-10} 
\cmidrule(l{1pt}r{1pt}){12-15}
& 
\makecell[c]{R-1} & \makecell[c]{R-2} & \makecell[c]{R-L} & \makecell[c]{M} &&
\makecell[c]{R-1} & \makecell[c]{R-2} & \makecell[c]{R-L} & \makecell[c]{M} &&
\makecell[c]{R-1} & \makecell[c]{R-2} & \makecell[c]{R-L} & \makecell[c]{M} \\
\midrule

& \multicolumn{14}{c}{Enron Emails} \\
\midrule
GPT-2 & 
1.00 & 1.00 & 1.00 & 1.00 &&
1.00 & 1.00 & 1.00 & 1.00 &&
0.04 & 0.00 & 0.04 & 0.03 \\

Llama-2-7b & 
0.96 & 0.95 & 0.96 & 0.97 &&
0.96 & 0.95 & 0.96 & 0.97 &&
0.02 & 0.00 & 0.02 & 0.04 \\

Llama-3-8b & 
0.90 & 0.89 & 0.90 & 0.94 &&
0.90 & 0.89 & 0.90 & 0.94 &&
0.07 & 0.00 & 0.06 & 0.05 \\

Llama-3.1-8b & 
0.90 & 0.89 & 0.90 & 0.94 &&
0.90 & 0.89 & 0.90 & 0.94 &&
0.08 & 0.00 & 0.07 & 0.06 \\

\midrule
& \multicolumn{14}{c}{Open Australian Legal Corpus} \\
\midrule

GPT-2 & 
0.90 & 0.89 & 0.89 & 0.87 &&
0.90 & 0.89 & 0.90 & 0.87 &&
0.05 & 0.00 & 0.04 & 0.05 \\

Llama-2-7b & 
0.88 & 0.88 & 0.88 & 0.86 &&
0.89 & 0.88 & 0.89 & 0.86 &&
0.03 & 0.00 & 0.03 & 0.05 \\

Llama-3-8b & 
0.87 & 0.85 & 0.86 & 0.86 &&
0.87 & 0.85 & 0.86 & 0.86 &&
0.08 & 0.00 & 0.07 & 0.05 \\

Llama-3.1-8b & 
0.87 & 0.85 & 0.86 & 0.86 &&
0.86 & 0.85 & 0.86 & 0.86 &&
0.08 & 0.00 & 0.07 & 0.04 \\

\midrule
& \multicolumn{14}{c}{AG News} \\
\midrule

GPT-2 & 
1.00 & 1.00 & 1.00 & 0.97 &&
1.00 & 1.00 & 1.00 & 0.97 &&
0.04 & 0.00 & 0.04 & 0.06 \\

Llama-2-7b & 
0.96 & 0.95 & 0.96 & 0.92 &&
0.96 & 0.95 & 0.96 & 0.92 &&
0.03 & 0.00 & 0.03 & 0.04 \\

Llama-3-8b & 
0.95 & 0.94 & 0.95 & 0.94 &&
0.95 & 0.94 & 0.95 & 0.94 &&
0.05 & 0.00 & 0.04 & 0.04 \\

Llama-3.1-8b & 
0.95 & 0.94 & 0.95 & 0.94 &&
0.95 & 0.94 & 0.95 & 0.94 &&
0.05 & 0.00 & 0.05 & 0.04 \\

\bottomrule
\end{tabular}
}
\end{threeparttable}
\end{table}

We use two extensively used gradient-level defenses against GIAs as baselines:
\begin{itemize}[left=0pt]
    \item \textbf{Gradient noise}: A defense that adds Gaussian or Laplacian noise to the gradients, suggested by Zhu \textit{et al.}~\cite{dlg} and Wei \textit{et al.}~\cite{wei2020framework}. We follow a standard DP-SGD~\cite{dpsgd} process, where the client normalizes or clips the gradients before adding the noise. 

    \item \textbf{Gradient pruning}: A defense that prunes part of the model parameter gradients~\cite{dlg}. We follow the gradient pruning variant in LAMP~\cite{lamp}, where a percentage of the gradients are randomly pruned. 
\end{itemize}

This selection of attack and defense methods provides a comprehensive evaluation of \codename's defense capabilities against early and advanced GIA techniques in comparison with existing gradient-level defense mechanisms.

\myparagraph{Attack details}
All attack experiments are conducted with default hyperparameter configurations. For each dataset, 64 samples are randomly sampled for the attacks, consistent with the implementation in GRAB~\cite{grab}. A batch size of 1 is used for attack experiments, as this configuration poses the greatest challenge to the defender~\cite{lamp, grab}. The attacked language model follows the setting outlined in LAMP~\cite{lamp}, with trainable embedding layers and no dropout applied. Additional assumptions regarding labels and sequence lengths, as described in Section~\ref{threat_model}, are also incorporated.

\myparagraph{Defense details}
\codename uses a hyperparameter combination with top-$k$ = 70, beam =1, and $\tau_o$ = 0.1, where the specific hyperparameters study can be found in Appendix~\ref{appendix:hyper}. For gradient noise and gradient pruning, the highest defense levels that exhibit acceptable utilities, reported by Feng \textit{et al.}~\cite{grab}, are applied, \textit{e.g.}, 0.05 noise level, and 0.99 prune ratio. This ensures that \codename's defense efficacy is fairly evaluated by avoiding suboptimal baseline settings.

\begin{table*}[!t]

\renewcommand{\arraystretch}{.8}
\setlength{\tabcolsep}{2pt}

\centering
\caption{Utility comparison across defenses. 
}
\label{table:utility_comparison}

\begin{threeparttable}
\resizebox{\linewidth}{!}{
\begin{tabular}{
l
rrrrrrrr p{0pt}
rrrrrrrr p{0pt}
rrrrrrrr p{0pt}
}
\toprule
& 
\multicolumn{8}{c}{SST-2} && 
\multicolumn{8}{c}{Tweet Sentiment Analysis} && 
\multicolumn{8}{c}{Yahoo Answers Topics} \\
\cmidrule(l{1pt}r{1pt}){2-9} 
\cmidrule(l{1pt}r{1pt}){11-18}
\cmidrule(l{1pt}r{1pt}){20-27}
& 
\multicolumn{2}{c}{Original} & \multicolumn{2}{c}{Noise} & \multicolumn{2}{c}{Pruning} & \multicolumn{2}{c}{\codename} && 
\multicolumn{2}{c}{Original} & \multicolumn{2}{c}{Noise} & \multicolumn{2}{c}{Pruning} & \multicolumn{2}{c}{\codename} &&
\multicolumn{2}{c}{Original} & \multicolumn{2}{c}{Noise} & \multicolumn{2}{c}{Pruning} & \multicolumn{2}{c}{\codename} \\
\cmidrule(l{1pt}r{1pt}){2-3} 
\cmidrule(l{1pt}r{1pt}){4-5} 
\cmidrule(l{1pt}r{1pt}){6-7}
\cmidrule(l{1pt}r{1pt}){8-9} 
\cmidrule(l{1pt}r{1pt}){11-12}
\cmidrule(l{1pt}r{1pt}){13-14}
\cmidrule(l{1pt}r{1pt}){15-16} 
\cmidrule(l{1pt}r{1pt}){17-18}
\cmidrule(l{1pt}r{1pt}){20-21}
\cmidrule(l{1pt}r{1pt}){22-23} 
\cmidrule(l{1pt}r{1pt}){24-25}
\cmidrule(l{1pt}r{1pt}){26-27}
& 
\multicolumn{1}{c}{Acc.} & \multicolumn{1}{c}{F1} & \multicolumn{1}{c}{Acc.} & \multicolumn{1}{c}{F1} & \multicolumn{1}{c}{Acc.} & \multicolumn{1}{c}{F1} & \multicolumn{1}{c}{Acc.} & \multicolumn{1}{c}{F1} &&
\multicolumn{1}{c}{Acc.} & \multicolumn{1}{c}{F1} & \multicolumn{1}{c}{Acc.} & \multicolumn{1}{c}{F1} & \multicolumn{1}{c}{Acc.} & \multicolumn{1}{c}{F1} & \multicolumn{1}{c}{Acc.} & \multicolumn{1}{c}{F1} &&
\multicolumn{1}{c}{Acc.} & \multicolumn{1}{c}{F1} & \multicolumn{1}{c}{Acc.} & \multicolumn{1}{c}{F1} & \multicolumn{1}{c}{Acc.} & \multicolumn{1}{c}{F1} & \multicolumn{1}{c}{Acc.} & \multicolumn{1}{c}{F1} \\
\midrule
BERT-base-uncased & 
0.92 & 0.93 & 0.56 & 0.61 & 0.90 & 0.90 & 0.82 & 0.85 &&
0.76 & 0.75 & 0.39 & 0.37 & 0.69 & 0.68 & 0.63 & 0.63 &&
0.65 & 0.65 & 0.55 & 0.55 & 0.60 & 0.59 & 0.60 & 0.60 \\

BERT-base-cased & 
0.89 & 0.90 & 0.68 & 0.76 & 0.71 & 0.78 & 0.84 & 0.85 &&
0.73 & 0.73 & 0.60 & 0.58 & 0.68 & 0.68 & 0.58 & 0.57 &&
0.65 & 0.65 & 0.54 & 0.53 & 0.59 & 0.58 & 0.60 & 0.60 \\

BERT-large-uncased & 
0.94 & 0.94 & 0.56 & 0.70 & 0.94 & 0.94 & 0.79 & 0.83 &&
0.74 & 0.73 & 0.69 & 0.68 & 0.70 & 0.69 & 0.61 & 0.59 &&
0.66 & 0.66 & 0.19 & 0.16 & 0.61 & 0.61 & 0.47 & 0.48 \\

BERT-large-cased & 
0.91 & 0.91 & 0.88 & 0.88 & 0.90 & 0.90 & 0.83 & 0.81 &&
0.75 & 0.73 & 0.56 & 0.54 & 0.67 & 0.67 & 0.63 & 0.60 &&
0.64 & 0.64 & 0.49 & 0.48 & 0.59 & 0.59 & 0.59 & 0.59 \\

\midrule
& 
\multicolumn{8}{c}{Enron Emails} && 
\multicolumn{8}{c}{Open Australian Legal Corpus} && 
\multicolumn{8}{c}{AG News} \\
\cmidrule(l{1pt}r{1pt}){2-9} 
\cmidrule(l{1pt}r{1pt}){11-18}
\cmidrule(l{1pt}r{1pt}){20-27}
& \multicolumn{2}{c}{Original} & \multicolumn{2}{c}{Noise} & \multicolumn{2}{c}{Pruning} & \multicolumn{2}{c}{\codename} && 
\multicolumn{2}{c}{Original} & \multicolumn{2}{c}{Noise} & \multicolumn{2}{c}{Pruning} & \multicolumn{2}{c}{\codename} &&
\multicolumn{2}{c}{Original} & \multicolumn{2}{c}{Noise} & \multicolumn{2}{c}{Pruning} & \multicolumn{2}{c}{\codename} \\
\cmidrule(l{1pt}r{1pt}){2-3} 
\cmidrule(l{1pt}r{1pt}){4-5} 
\cmidrule(l{1pt}r{1pt}){6-7}
\cmidrule(l{1pt}r{1pt}){8-9} 
\cmidrule(l{1pt}r{1pt}){11-12}
\cmidrule(l{1pt}r{1pt}){13-14}
\cmidrule(l{1pt}r{1pt}){15-16} 
\cmidrule(l{1pt}r{1pt}){17-18}
\cmidrule(l{1pt}r{1pt}){20-21}
\cmidrule(l{1pt}r{1pt}){22-23} 
\cmidrule(l{1pt}r{1pt}){24-25}
\cmidrule(l{1pt}r{1pt}){26-27}
& 
\multicolumn{1}{c}{Loss} & \multicolumn{1}{c}{PPL} & \multicolumn{1}{c}{Loss} & \multicolumn{1}{c}{PPL} & \multicolumn{1}{c}{Loss} & \multicolumn{1}{c}{PPL} & \multicolumn{1}{c}{Loss} & \multicolumn{1}{c}{PPL} &&
\multicolumn{1}{c}{Loss} & \multicolumn{1}{c}{PPL} & \multicolumn{1}{c}{Loss} & \multicolumn{1}{c}{PPL} & \multicolumn{1}{c}{Loss} & \multicolumn{1}{c}{PPL} & \multicolumn{1}{c}{Loss} & \multicolumn{1}{c}{PPL} &&
\multicolumn{1}{c}{Loss} & \multicolumn{1}{c}{PPL} & \multicolumn{1}{c}{Loss} & \multicolumn{1}{c}{PPL} & \multicolumn{1}{c}{Loss} & \multicolumn{1}{c}{PPL} & \multicolumn{1}{c}{Loss} & \multicolumn{1}{c}{PPL} \\

\midrule

GPT-2 & 
2.40 & 11.05 & \multicolumn{1}{c}{---} & \multicolumn{1}{c}{---} & 2.40 & 10.82 & 2.95 & 19.12 &&
1.78 & 5.91 & \multicolumn{1}{c}{---} & \multicolumn{1}{c}{---} & 1.74 & 5.71 & 2.01 & 7.46 &&
2.20 & 9.07 & \multicolumn{1}{c}{---} & \multicolumn{1}{c}{---} & 2.91 & 8.90 & 2.56 & 12.91 \\

Llama-2-7b & 
1.23 & 3.42 & \multicolumn{1}{c}{---} & \multicolumn{1}{c}{---} & 1.22 & 3.40 & 1.73 & 5.64 &&
1.32 & 3.76 & \multicolumn{1}{c}{---} & \multicolumn{1}{c}{---} & 1.32 & 3.75 & 1.70 & 5.45 &&
1.57 & 4.79 & \multicolumn{1}{c}{---} & \multicolumn{1}{c}{---} & 1.56 & 4.76 & 1.94 & 6.95 \\

Llama-3-8b & 
1.86 & 6.45 & \multicolumn{1}{c}{---} & \multicolumn{1}{c}{---} & 1.85 & 6.39 & 2.40 & 11.03 &&
1.74 & 5.72 & \multicolumn{1}{c}{---} & \multicolumn{1}{c}{---} & 1.73 & 5.65 & 2.09 & 8.12 &&
1.75 & 5.74 & \multicolumn{1}{c}{---} & \multicolumn{1}{c}{---} & 1.74 & 5.69 & 2.04 & 7.65 \\

Llama-3.1-8b & 
1.88 & 6.53 & \multicolumn{1}{c}{---} & \multicolumn{1}{c}{---} & 1.89 & 6.39 & 2.39 & 10.89 &&
1.75 & 5.76 & \multicolumn{1}{c}{---} & \multicolumn{1}{c}{---} & 1.74 & 5.71 & 2.09 & 8.08 &&
1.75 & 5.75 & \multicolumn{1}{c}{---} & \multicolumn{1}{c}{---} & 1.74 & 5.70 & 2.05 & 7.75 \\

\bottomrule
\end{tabular}
}

\begin{tablenotes}
    \footnotesize
    \item ---: At the time of writing, the implementation of DAGER is encountering an error with the gradient noise defense. To align with the attack efficacy evaluated, we do not \\ evaluate the utility of gradient noise.
\end{tablenotes}

\end{threeparttable}
\end{table*}

\myparagraph{Supported configurations}
Since some attack methods are incompatible with certain models and defenses, only supported configurations are reported. These include DLG, TAG, LAMP, and GRAB applied to the four BERT models with all defenses, as well as DAGER applied to BERT-base-uncased, GPT-2, Llama-2-7b, Llama-3-8b, and Llama-3.1-8b with gradient pruning and \codename. A detailed description of the unsupported experiments due to incompatibility is provided in Appendix~\ref{appendix:unsupported}.

\myparagraph{Metrics}
We use ROUGE and METEOR scores to measure the similarity between the recovered sentences and the original sentences as the attack efficacy. For presentation purposes, we report the defense efficacy for classification models, which is (1 minus attack efficacy), while we report the attack efficacy for generative models.

\myparagraph{Results} The defense efficacy for classification models is illustrated in Figure~\ref{fig:classification_gia}, where greater radar chart coverage indicates stronger defense performance. Correspondingly, the attack results for generative models are presented in Table~\ref{table:generative_gia}.
Across all evaluated models and datasets, most attacks recover a significant portion of original data when defenses are absent, with R-1 scores ranging from 0.11 to 1.00 and METEOR scores from 0.06 to 0.98. Baseline defenses like gradient noise and gradient pruning effectively mitigate earlier attacks (\textit{e.g.}, DLG, TAG) in most cases, achieving R-1 scores between 0.00 and 0.37 and METEOR scores between 0.00 and 0.23. However, these defenses struggle against advanced methods (\textit{e.g.}, GRAB, DAGER), where R-1 scores range from 0.13 to 1.00 and METEOR scores from 0.10 to 0.98, with DAGER achieving near-perfect data recovery. In contrast, \codename significantly reduces attack success across all methods, with R-1 scores ranging from 0.01 to 0.08 and METEOR scores from 0.01 to 0.06. This represents up to 98\% more defended tokens compared to the strongest baseline defense. These results highlight \codename's superior defense capabilities, outperforming other methods and providing robust protection against all evaluated GIA methods.

Table~\ref{table:utility_comparison} provides a summary of utility-privacy trade-offs under various defenses. In the absence of any defense, classification models achieve accuracies and F1 scores ranging from 0.64 to 0.94, while generative models exhibit losses between 1.23 and 2.40 and perplexities from 3.42 to 11.05. When gradient noise is applied, classification model accuracies drop to a range of 0.19–0.88, with F1 scores ranging from 0.16 to 0.88. Gradient pruning results in classification accuracies of 0.59–0.94 and F1 scores of 0.58–0.94, while generative model losses range from 1.22 to 2.91 and perplexities from 3.40 to 10.82. With \codename applied, classification model accuracies fall between 0.47 and 0.84, and F1 scores between 0.48 and 0.85, while generative model losses range from 1.70 to 2.95 and perplexities from 5.45 to 19.12. In most cases, these defenses achieve comparable utility levels. However, considering the privacy preservation ability, \codename achieves the best utility-privacy trade-off. It preserves model utility comparable to non-defended models while significantly diminishing attack efficacy, demonstrating its superior ability to balance privacy protection and utility preservation.

\subsection{Ablation Study}
\label{sec:ablation}
We perform an ablation study to isolate and evaluate the effectiveness of \codename's individual components. Specifically, we investigate the impact of the similarity heuristics used during the searching step, and we separately analyze the role of the selection step. 
\myparagraph{Ablation details} For the searching step, we assess the contribution of the three similarity heuristics. This involves systematically disabling each heuristic and observing its impact on the degree of data obfuscation (privacy). For the selection step, to understand the importance of the selection mechanism in maintaining model utility, we compare \codename's selection strategy against two simplified baselines, which are random selection and no selection. For random selection, the shadow tokens are randomly selected from candidates identified in the searching step, disregarding their alignment with the internal model outputs. For no selection, tokens are directly obfuscated with the candidate tokens closest to the original tokens in the embedding space.

\myparagraph{Model and dataset} We evaluate BERT-base-uncased on the SST-2 dataset to demonstrate the effectiveness of the various components of \codename.

\myparagraph{Metrics} To align with the main experiments, we use accuracy and F1 scores to measure the model utility, and use ROUGE score and METEOR score to measure the data obfuscation level.

\myparagraph{Results} The results of the ablation study are summarized in Table~\ref{table:ablation} in Appendix~\ref{appendix: ablation}. For the similarity heuristics, enabling all heuristics or selectively disabling any individual heuristic leads to nearly zero ROUGE and METEOR scores. When disabling two heuristics simultaneously, higher METEOR scores are observed (up to 0.1). Moreover, completely disabling all heuristics results in a significant increase in the METEOR scores of up to 0.2. This indicates that similarity heuristics work effectively together to reduce semantic overlap, thus enhancing privacy protection. Regarding the selection methods, applying random search yields only low utility, whereas applying no search achieves a moderate level of utility. In contrast, employing our optimization method attains a notably high utility level. This demonstrates the efficacy of our optimization approach in maintaining model utility despite data obfuscation.

\begin{table*}[!htbp]

\renewcommand{\arraystretch}{.5}
\setlength{\tabcolsep}{2pt}

\centering
\caption{Attack efficacy of adaptive attacks on SST-2 dataset.}
\label{table:adaptive_attack_partial}

\begin{threeparttable}
\scriptsize

\resizebox{\linewidth}{!}{

\begin{tabular}{
l
rrrr p{1pt}
rrrr p{1pt}
rrrr p{1pt}
rrrr p{1pt}
rrrr
}
\toprule
& 
\multicolumn{4}{c}{Sampling} &&
\multicolumn{4}{c}{Max Similarity} &&
\multicolumn{4}{c}{Median Similarity} &&
\multicolumn{4}{c}{Mean Similarity} &&
\multicolumn{4}{c}{LLM generation} \\
\cmidrule(l{1pt}r{1pt}){2-5}
\cmidrule(l{1pt}r{1pt}){7-10} 
\cmidrule(l{1pt}r{1pt}){12-15}
\cmidrule(l{1pt}r{1pt}){17-20}
\cmidrule(l{1pt}r{1pt}){22-25}
& 
\multicolumn{1}{c}{R-1} & \multicolumn{1}{c}{R-2} & \multicolumn{1}{c}{R-L} & \multicolumn{1}{c}{M} &&
\multicolumn{1}{c}{R-1} & \multicolumn{1}{c}{R-2} & \multicolumn{1}{c}{R-L} & \multicolumn{1}{c}{M} &&
\multicolumn{1}{c}{R-1} & \multicolumn{1}{c}{R-2} & \multicolumn{1}{c}{R-L} & \multicolumn{1}{c}{M} &&
\multicolumn{1}{c}{R-1} & \multicolumn{1}{c}{R-2} & \multicolumn{1}{c}{R-L} & \multicolumn{1}{c}{M} &&
\multicolumn{1}{c}{R-1} & \multicolumn{1}{c}{R-2} & \multicolumn{1}{c}{R-L} & \multicolumn{1}{c}{M} \\

\midrule
& 
\multicolumn{24}{c}{{SST-2}} \\
\midrule

BERT-base-uncased &
0.20 & 0.05 & 0.20 & 0.14 &&
0.20 & 0.04 & 0.20 & 0.15 &&
0.20 & 0.03 & 0.19 & 0.14 &&
0.21 & 0.04 & 0.21 & 0.17 &&
0.10 & 0.01 & 0.09 & 0.08 \\

BERT-base-cased &
0.09 & 0.01 & 0.08 & 0.06 &&
0.07 & 0.00 & 0.07 & 0.06 &&
0.07 & 0.00 & 0.07 & 0.05 &&
0.09 & 0.01 & 0.08 & 0.05 &&
0.07 & 0.01 & 0.06 & 0.05 \\

BERT-large-uncased &
0.12 & 0.02 & 0.12 & 0.08 &&
0.15 & 0.02 & 0.14 & 0.11 &&
0.11 & 0.01 & 0.10 & 0.06 &&
0.15 & 0.02 & 0.15 & 0.11 &&
0.08 & 0.01 & 0.07 & 0.07 \\

BERT-large-cased &
0.08 & 0.01 & 0.07 & 0.05 &&
0.14 & 0.02 & 0.13 & 0.09 &&
0.06 & 0.00 & 0.06 & 0.04 &&
0.10 & 0.01 & 0.09 & 0.07 &&
0.10 & 0.01 & 0.08 & 0.08 \\

RoBERTa-base &
0.05 & 0.00 & 0.05 & 0.03 &&
0.05 & 0.00 & 0.05 & 0.03 &&
0.05 & 0.00 & 0.04 & 0.03 &&
0.04 & 0.00 & 0.04 & 0.03 &&
0.09 & 0.01 & 0.07 & 0.07 \\

RoBERTa-large &
0.04 & 0.00 & 0.04 & 0.03 &&
0.07 & 0.00 & 0.06 & 0.04 &&
0.05 & 0.00 & 0.05 & 0.04 &&
0.07 & 0.00 & 0.06 & 0.04 &&
0.09 & 0.01 & 0.08 & 0.08 \\

DeBERTa-v3-small &
0.07 & 0.00 & 0.07 & 0.04 &&
0.11 & 0.01 & 0.11 & 0.07 &&
0.07 & 0.00 & 0.06 & 0.05 &&
0.14 & 0.02 & 0.14 & 0.10 &&
0.11 & 0.01 & 0.10 & 0.09 \\

DeBERTa-v3-base &
0.01 & 0.00 & 0.01 & 0.01 &&
0.10 & 0.01 & 0.10 & 0.07 &&
0.01 & 0.00 & 0.01 & 0.00 &&
0.09 & 0.00 & 0.08 & 0.06 &&
0.06 & 0.00 & 0.06 & 0.05 \\

DeBERTa-v3-large &
0.05 & 0.00 & 0.05 & 0.04 &&
0.17 & 0.02 & 0.16 & 0.14 &&
0.04 & 0.00 & 0.04 & 0.03 &&
0.18 & 0.03 & 0.18 & 0.15 &&
0.18 & 0.04 & 0.17 & 0.16 \\

\bottomrule
\end{tabular}
}

\end{threeparttable}
\end{table*}

\subsection{Adaptive Attacks}
\label{sec:adaptive_attacks}

Beyond standard GIAs, we evaluate the robustness of \codename against adaptive attacks, assuming an attacker with significantly enhanced capabilities. 

\myparagraph{Adversary capabilities}
In this scenario, the attacker is able to fully recover the obfuscated data through GIAs, given that the objective of the current GIAs is to recover the data that participates in the training. The attacker also has complete knowledge of the defense mechanism, including details such as the top-$k$ parameter, the number of beams, and the overlap ratio $\tau_o$. Specifically, the attacker understands that each token $t$ in the obfuscated data is a shadow token of the original token $t_{ori}$ that resides within the shadow token set $\mathcal{S}_{t_{ori}}$.

To exploit this knowledge, the attacker can extract the set $\mathcal{T}_{ori}$, comprising all tokens for which $t$ appears as a shadow token, and then compute the cosine similarities between $t$ and tokens in $\mathcal{T}_{ori}$. Using this information, the attacker can design strategies to reverse-engineer the obfuscation process, attempting to recover the original token $t_{ori}$ from the shadow token $t$. This setup represents a challenging attack scenario that assesses \codename's robustness under adaptive adversarial conditions. 

\myparagraph{Adaptive attack strategies}
We evaluate the robustness of \codename against five potential adaptive attack strategies, detailed as follows:
\begin{itemize}[left=0pt]
    \item \textbf{Sampling}: The attacker uses the computed similarities between $t$ and $\mathcal{T}_{ori}$ as probabilities to randomly sample a token from $\mathcal{T}_{ori}$ to replace $t$.
    \item \textbf{Maximum similarity}: The attacker selects the token with the highest similarity from $\mathcal{T}_{ori}$ to replace $t$.
    \item \textbf{Median similarity}: The attacker selects the token with the median similarity from $\mathcal{T}_{ori}$ to replace $t$.
    \item \textbf{Mean embedding}: The attacker computes the mean embedding of all tokens in $\mathcal{T}_{ori}$ and selects the token with the highest similarity to the mean embedding from $\mathcal{T}_{ori}$ to replace $t$.
    \item \textbf{LLM generation}: The attacker identifies a reference dataset from a similar domain and applies \codename to create obfuscated data. A generative LLM is then trained on the obfuscated data to generate the original data. Details of the experimental settings of this adaptive attack can be found in Appendix~\ref{appendix:LLM_gen_attack}.
\end{itemize}

\myparagraph{Metrics}
We use ROUGE and METEOR scores to measure the similarity between the recovered sentences and the original sentences as the attack efficacy.

\myparagraph{Results}
The results of adaptive attacks on the SST-2 dataset are presented in Table~\ref{table:adaptive_attack_partial} for demonstration, and the full results are presented in Table~\ref{table:adaptive_attack} in Appendix~\ref{appendix: adaptive}. Although most adaptive attacks demonstrate increased attack efficacy compared to standard GIAs, the overall data recovery remains relatively low, with R-1 scores ranging from 0.00 to 0.29 and METEOR scores from 0.00 to 0.25. \codename effectively withstands even challenging adaptive attacks and exhibits exceptional robustness, where a powerful adversary possesses complete knowledge of the applied defense methods. This effective mitigation directly relates to our theoretical analysis, where it shows that the gradient difference between obfuscated and original data grows proportionally larger than their loss difference, making the original tokens indistinguishable.
\section{Related work}
We review GIAs, defenses, and text modification techniques.

\subsection{Existing GIAs}
Collaborative learning models are vulnerable to privacy attacks. Beyond membership inference~\cite{ma2023loden, shokri2017membership, mireshghallah2022quantifying}, GIAs exploit shared gradients to reconstruct private training data, achieving strong results in vision tasks using labels, priors, and batch statistics~\cite{dlg, idlg, ig, april, stg, gip}. More recently, these methods have been extended to text data of language models~\cite{dlg, tag, film, lamp, grab, decep, dager}.
Early work on language models directly applies continuous optimization from image-based GIAs. DLG~\cite{dlg} first demonstrates partial recovery from BERT, while TAG~\cite{tag} improves token recovery with $L_1$ regularization. However, continuous optimization alone struggles with the token ordering of the reconstructed sequences. Discrete optimization methods address this: FILM~\cite{film} uses GPT-2’s memorization, LAMP~\cite{lamp} alternates gradient descent and random search, GRAB~\cite{grab} extends to practical FL with dropout and beam search, and DAGER~\cite{dager} achieves near-perfect recovery via analyzing attention ranks.
Beyond these passive attacks, DECEPTICONS~\cite{decep} actively injects malicious updates to extract high-fidelity text, though its resilience to detection remains unclear.
\subsection{Defenses against GIAs}
Defenses generally aim to disrupt gradient patterns. Gradient noise injection~\cite{dlg, wei2020framework} or gradient pruning~\cite{dlg} degrade recovery by perturbing shared gradients. Freezing embeddings~\cite{film} and enabling dropout provide lightweight protection for language models. Approaches that generate surrogate data, such as Refiner~\cite{fan2022refiner}, which mimics gradient behavior in vision models but is impractical for real-time training.

\subsection{Text Modifications}
Text modification techniques have been explored. 
Defensively, approaches include character-level edits for adversarial training~\cite{ma2023formalizing}, token substitutions with differential privacy guarantees~\cite{chen-etal-2023-customized}, and rewriting via DP-BART~\cite{igamberdiev-habernal-2023-dp}. 
Pape \textit{et al.}~\cite{pape2024prompt} obfuscate prompts with hard or soft transformation that can lead to similar outputs of the models to prevent prompt extraction. Wang \textit{et al.}~\cite{wang2025ai} modulates the model output by adding or subtracting noise to the output logits.

\section{Conclusion}
In this work, we propose \codename, a novel token-level defense mechanism against GIAs by decoupling the inherent connection across gradient, embedding, and token spaces. \codename is built upon an important insight: due to the large scale of the token space, there exist semantically distinct yet embedding-proximate tokens that can serve as the shadow substitutes of the original tokens. By systematically searching and selecting these shadow tokens to obfuscate the original data, it achieves exceptional privacy protection while preserving the critical features required for training. \codename demonstrates a paradigm shift from the gradient level to the token level for effective mitigation against GIAs.

\newpage
\bibliographystyle{ACM-Reference-Format}
\bibliography{references}

@inproceedings{fl,
  title={Communication-efficient learning of deep networks from decentralized data},
  author={McMahan, Brendan and Moore, Eider and Ramage, Daniel and Hampson, Seth and y Arcas, Blaise Aguera},
  booktitle={Proceedings of the 2017 International Conference on Artificial Intelligence and Statistics (AISTATS)},
  pages={1273--1282},
  year={2017},
}

@article{sst2_dataset,
  title={The Stanford Sentiment Treebank},
  author={Socher, Richard and Perelygin, Alex and Wu, Jean and Chuang, Jason and Manning, Christopher D and Ng, Andrew Y and Potts, Christopher},
  journal={Proceedings of the Conference on Empirical Methods in Natural Language Processing (EMNLP)},
  year={2013},
  pages={1631--1642}
}

@article{rotten_tomatoes_dataset,
  title={Rotten Tomatoes Movie Reviews Dataset},
  author={Pang, Bo and Lee, Lillian},
  journal={Proceedings of the Annual Meeting of the Association for Computational Linguistics (ACL)},
  year={2005},
  pages={79--86}
}

@inproceedings{tweet_sentiment_dataset,
  title={Sentiment analysis on microblogging platforms: A comprehensive review},
  author={Pak, Alexander and Paroubek, Patrick},
  booktitle={Proceedings of the International Conference on Weblogs and Social Media (ICWSM)},
  year={2010},
  pages={145--152}
}

@misc{yahoo_answers_dataset,
      title={Yahoo Answers Topics}, 
        howpublished={\url{https://huggingface.co/datasets/community-datasets/yahoo_answers_topics}},
      author={Huggingface},
      year={2024},
note={Accessed: October, 2024}
}

@inproceedings{enron_emails_dataset,
  title={The Enron Email Dataset: A New Benchmark for Text Analysis},
  author={Cohen, William W},
  booktitle={Conference on Email and Anti-Spam (CEAS)},
  year={2004}
}

@inproceedings{open_australian_legal_corpus,
  title={Building an Open Australian Legal Corpus},
  author={Cabrera, Martha and Baldwin, Timothy and Verspoor, Karin},
  booktitle={Proceedings of the Australasian Language Technology Workshop (ALTW)},
  year={2016},
  pages={58--66}
}

@article{ag_news_dataset,
  title={AG News Corpus: A Large-scale News Categorization Dataset},
  author={Zhang, Xiang and Zhao, Junbo and LeCun, Yann},
  journal={Proceedings of Advances in Neural Information Processing Systems (NeurIPS)},
  year={2015},
  pages={1487--1495}
}

@inproceedings{cai2023federated,
  title={Federated few-shot learning for mobile NLP},
  author={Cai, Dongqi and Wang, Shangguang and Wu, Yaozong and Lin, Felix Xiaozhu and Xu, Mengwei},
  booktitle={Proceedings of the 2023 Annual International Conference on Mobile Computing and Networking (MobiCom)},
  pages={1--17},
  year={2023}
}

@article{tian2022fedbert,
  title={Fedbert: When federated learning meets pre-training},
  author={Tian, Yuanyishu and Wan, Yao and Lyu, Lingjuan and Yao, Dezhong and Jin, Hai and Sun, Lichao},
  journal={ACM Transactions on Intelligent Systems and Technology (TIST)},
  volume={13},
  number={4},
  pages={1--26},
  year={2022},
  publisher={ACM New York, NY}
}

@misc{FedML,
  author = {FedML},
  title = {Releasing FedLLM: Build Your Own Large Language Models on Proprietary Data using the FedML Platform},
  year = {2023},
  howpublished = {\url{https://blog.fedml.ai/releasing-fedllm-build-your-own-large-language-models-on-proprietary-data-using-the-fedml-platform/}},
note={Accessed: January, 2024}
}

@inproceedings{dlg,
  title={Deep leakage from gradients},
  author={Zhu, Ligeng and Liu, Zhijian and Han, Song},
  booktitle={the 2019 Advances in Neural Information Processing Systems (NeurIPS)},
  year={2019}
}

@article{idlg,
  title={idlg: Improved deep leakage from gradients},
  author={Zhao, Bo and Mopuri, Konda Reddy and Bilen, Hakan},
  journal={arXiv preprint arXiv:2001.02610},
  year={2020}
}

@inproceedings{gip,
  title={Gradient inversion with generative image prior},
  author={Jeon, Jinwoo and Lee, Kangwook and Oh, Sewoong and Ok, Jungseul and others},
  booktitle={the 2021 Advances in Neural Information Processing Systems (NeurIPS)},
  pages={29898--29908},
  year={2021}
}

@inproceedings{stg,
  title={See through gradients: Image batch recovery via gradinversion},
  author={Yin, Hongxu and Mallya, Arun and Vahdat, Arash and Alvarez, Jose M and Kautz, Jan and Molchanov, Pavlo},
  booktitle={Proceedings of the 2021 IEEE/CVF Conference on Computer Vision and Pattern Recognition (CVPR)},
  pages={16337--16346},
  year={2021}
}

@inproceedings{april,
  title={APRIL: Finding the Achilles' Heel on Privacy for Vision Transformers},
  author={Lu, Jiahao and Zhang, Xi Sheryl and Zhao, Tianli and He, Xiangyu and Cheng, Jian},
  booktitle={Proceedings of the 2022 IEEE/CVF Conference on Computer Vision and Pattern Recognition (CVPR)},
  pages={10051--10060},
  year={2022}
}

@inproceedings{ig,
  title={Inverting gradients-how easy is it to break privacy in federated learning?},
  author={Geiping, Jonas and Bauermeister, Hartmut and Dr{\"o}ge, Hannah and Moeller, Michael},
  booktitle={the 2020 Advances in Neural Information Processing Systems (NeurIPS)},
  pages={16937--16947},
  year={2020}
}

@article{tag,
  title={Tag: Gradient attack on transformer-based language models},
  author={Deng, Jieren and Wang, Yijue and Li, Ji and Shang, Chao and Liu, Hang and Rajasekaran, Sanguthevar and Ding, Caiwen},
  journal={arXiv preprint arXiv:2103.06819},
  year={2021}
}

@inproceedings{lamp,
  title={Lamp: Extracting text from gradients with language model priors},
  author={Balunovic, Mislav and Dimitrov, Dimitar and Jovanovi{\'c}, Nikola and Vechev, Martin},
  booktitle={the 2022 Advances in Neural Information Processing Systems (NeurIPS)},
  pages={7641--7654},
  year={2022}
}

@inproceedings{film,
  title={Recovering private text in federated learning of language models},
  author={Gupta, Samyak and Huang, Yangsibo and Zhong, Zexuan and Gao, Tianyu and Li, Kai and Chen, Danqi},
  booktitle={the 2022 Advances in Neural Information Processing Systems (NeurIPS)},
  pages={8130--8143},
  year={2022}
}

@inproceedings{decep,
  title={Decepticons: Corrupted transformers breach privacy in federated learning for language models},
  author={Fowl, Liam and Geiping, Jonas and Reich, Steven and Wen, Yuxin and Czaja, Wojtek and Goldblum, Micah and Goldstein, Tom},
  booktitle={Proceedings of the 2022 International Conference on Learning Representations (ICLR)},
  year={2022}
}

@inproceedings{grab,
  title={Uncovering Gradient Inversion Risks in Practical Language Model Training},
  author={Feng, Xinguo and Ma, Zhongkui and Wang, Zihan and Chegne, Eu Joe and Ma, Mengyao and Abuadbba, Alsharif and Bai, Guangdong},
  booktitle={Proceedings of the 2024 ACM SIGSAC Conference on Computer and Communications Security (CCS '24)},
  year={2024},
  month={October 14--18},
  address={Salt Lake City, UT, USA},
  publisher={ACM},
  location={New York, NY, USA},
  pages={15},
  doi={10.1145/3658644.3690292},
  url={https://doi.org/10.1145/3658644.3690292}
}

@inproceedings{dager,
  title={DAGER: Exact Gradient Inversion for Large Language Models},
  author={Petrov, Ivo and Dimitrov, Dimitar I and Baader, Maximilian and M{\"u}ller, Mark Niklas and Vechev, Martin},
  booktitle={the 2024 Advances in Neural Information Processing Systems (NeurIPS)},
  year={2024}
}

@inproceedings{dpsgd,
  title={Deep learning with differential privacy},
  author={Abadi, Martin and Chu, Andy and Goodfellow, Ian and McMahan, H Brendan and Mironov, Ilya and Talwar, Kunal and Zhang, Li},
  booktitle={Proceedings of the 2016 ACM SIGSAC conference on computer and communications security (CCS)},
  pages={308--318},
  year={2016}
}

@article{wei2020framework,
  title={A framework for evaluating gradient leakage attacks in federated learning},
  author={Wei, Wenqi and Liu, Ling and Loper, Margaret and Chow, Ka-Ho and Gursoy, Mehmet Emre and Truex, Stacey and Wu, Yanzhao},
  journal={arXiv preprint arXiv:2004.10397},
  year={2020}
}

@article{bert,
  title={Bert: Pre-training of deep bidirectional transformers for language understanding},
  author={Devlin, Jacob and Chang, Ming-Wei and Lee, Kenton and Toutanova, Kristina},
  journal={arXiv preprint arXiv:1810.04805},
  year={2018}
}

@article{roberta,
  title={Roberta: A robustly optimized bert pretraining approach},
  author={Liu, Yinhan and Ott, Myle and Goyal, Naman and Du, Jingfei and Joshi, Mandar and Chen, Danqi and Levy, Omer and Lewis, Mike and Zettlemoyer, Luke and Stoyanov, Veselin},
  journal={arXiv preprint arXiv:1907.11692},
  year={2019}
}

@inproceedings{he2020deberta,
  title={Deberta: Decoding-enhanced bert with disentangled attention},
  author={He, Pengcheng and Liu, Xiaodong and Gao, Jianfeng and Chen, Weizhu},
  booktitle={Proceedings of the 2022 International Conference on Learning Representations (ICLR)},
  year={2020}
}

@article{gpt2,
  title={Language models are unsupervised multitask learners},
  author={Radford, Alec and Wu, Jeffrey and Child, Rewon and Luan, David and Amodei, Dario and Sutskever, Ilya and others},
  journal={OpenAI blog},
  volume={1},
  number={8},
  pages={9},
  year={2019}
}

@article{touvron2023llama,
  title={Llama: Open and efficient foundation language models},
  author={Touvron, Hugo and Lavril, Thibaut and Izacard, Gautier and Martinet, Xavier and Lachaux, Marie-Anne and Lacroix, Timoth{\'e}e and Rozi{\`e}re, Baptiste and Goyal, Naman and Hambro, Eric and Azhar, Faisal and others},
  journal={arXiv preprint arXiv:2302.13971},
  year={2023}
}

@article{team2024gemma,
  title={Gemma: Open models based on gemini research and technology},
  author={Team, Gemma and Mesnard, Thomas and Hardin, Cassidy and Dadashi, Robert and Bhupatiraju, Surya and Pathak, Shreya and Sifre, Laurent and Rivi{\`e}re, Morgane and Kale, Mihir Sanjay and Love, Juliette and others},
  journal={arXiv preprint arXiv:2403.08295},
  year={2024}
}

@article{sentiment,
  title={Sentiment analysis algorithms and applications: A survey},
  author={Medhat, Walaa and Hassan, Ahmed and Korashy, Hoda},
  journal={Ain Shams engineering journal},
  volume={5},
  number={4},
  pages={1093--1113},
  year={2014},
  publisher={Elsevier}
}

@article{ner,
  title={A survey on deep learning for named entity recognition},
  author={Li, Jing and Sun, Aixin and Han, Jianglei and Li, Chenliang},
  journal={IEEE transactions on knowledge and data engineering},
  volume={34},
  number={1},
  pages={50--70},
  year={2020},
  publisher={IEEE}
}

@book{koehn2009statistical,
  title={Statistical machine translation},
  author={Koehn, Philipp},
  year={2009},
  publisher={Cambridge University Press}
}

@inproceedings{transformer,
  title={Attention is all you need},
  author={Vaswani, Ashish and Shazeer, Noam and Parmar, Niki and Uszkoreit, Jakob and Jones, Llion and Gomez, Aidan N and Kaiser, {\L}ukasz and Polosukhin, Illia},
  booktitle={the 2017 Advances in Neural Information Processing Systems (NeurIPS)},
  year={2017}
}

@inproceedings{hu2021lora,
  title={Lora: Low-rank adaptation of large language models},
  author={Hu, Edward J and Shen, Yelong and Wallis, Phillip and Allen-Zhu, Zeyuan and Li, Yuanzhi and Wang, Shean and Wang, Lu and Chen, Weizhu},
  booktitle={Proceedings of the 2022 International Conference on Learning Representations (ICLR)},
  year={2022}
}

@inproceedings{rouge,
    title = "{ROUGE}: A Package for Automatic Evaluation of Summaries",
    author = "Lin, Chin-Yew",
    booktitle = "Text Summarization Branches Out",
    month = Jul,
    year = "2004",
    pages = "74--81",
}

@inproceedings{banerjee2005meteor,
  title={METEOR: An automatic metric for MT evaluation with improved correlation with human judgments},
  author={Banerjee, Satanjeev and Lavie, Alon},
  booktitle={Proceedings of the acl workshop on intrinsic and extrinsic evaluation measures for machine translation and/or summarization},
  pages={65--72},
  year={2005}
}

@inproceedings{zhang2019bertscore,
  title={Bertscore: Evaluating text generation with bert},
  author={Zhang, Tianyi and Kishore, Varsha and Wu, Felix and Weinberger, Kilian Q and Artzi, Yoav},
  booktitle={Proceedings of the 2020 International Conference on Learning Representations (ICLR)},
  year={2020}
}

@misc{fine_persona,
      title={
FinePersonas-Conversations-Email-Summaries}, 
        howpublished={\url{https://huggingface.co/datasets/argilla/FinePersonas-Conversations-Email-Summaries}},
      author={Huggingface},
      year={2024},
note={Accessed: April, 2025}
}

@misc{legal_task,
      title={legal-task}, 
        howpublished={\url{https://huggingface.co/datasets/darrow-ai/legal-task}},
      author={Huggingface},
      year={2024},
note={Accessed: April, 2025}
}

@misc{news_dataset,
      title={news-dataset-seq}, 
        howpublished={\url{https://huggingface.co/datasets/prithivMLmods/News-Dataset-Seq}},
      author={Huggingface},
      year={2024},
note={Accessed: April, 2025}
}

@inproceedings{ma2023loden,
  title={Loden: Making every client in federated learning a defender against the poisoning membership inference attacks},
  author={Ma, Mengyao and Zhang, Yanjun and Arachchige, Pathum Chamikara Mahawaga and Zhang, Leo Yu and Chhetri, Mohan Baruwal and Bai, Guangdong},
  booktitle={Proceedings of the 2023 ACM Asia Conference on Computer and Communications Security},
  pages={122--135},
  year={2023}
}

@inproceedings{shokri2017membership,
  title={Membership inference attacks against machine learning models},
  author={Shokri, Reza and Stronati, Marco and Song, Congzheng and Shmatikov, Vitaly},
  booktitle={Proceedings of the 2017 IEEE Symposium on Security and Privacy (SP)},
  pages={3--18},
  year={2017},
}

@inproceedings{mireshghallah2022quantifying,
  title={Quantifying Privacy Risks of Masked Language Models Using Membership Inference Attacks},
  author={Mireshghallah, Fatemehsadat and Goyal, Kartik and Uniyal, Archit and Berg-Kirkpatrick, Taylor and Shokri, Reza},
  booktitle={Proceedings of the 2022 Conference on Empirical Methods in Natural Language Processing (EMNLP)},
  pages={8332--8347},
  year={2022}
}

@article{fan2022refiner,
  title={Refiner: Data Refining against Gradient Leakage Attacks in Federated Learning},
  author={Fan, Mingyuan and Chen, Cen and Wang, Chengyu and Li, Xiaodan and Zhou, Wenmeng and Huang, Jun},
  journal={arXiv preprint arXiv:2212.02042},
  year={2022}
}

@inproceedings{chen-etal-2023-customized,
    title = "A Customized Text Sanitization Mechanism with Differential Privacy",
    author = "Chen, Sai  and
      Mo, Fengran  and
      Wang, Yanhao  and
      Chen, Cen  and
      Nie, Jian-Yun  and
      Wang, Chengyu  and
      Cui, Jamie",
    editor = "Rogers, Anna  and
      Boyd-Graber, Jordan  and
      Okazaki, Naoaki",
    booktitle = "Findings of the Association for Computational Linguistics: ACL 2023",
    month = jul,
    year = "2023",
    address = "Toronto, Canada",
    publisher = "Association for Computational Linguistics",
    url = "https://aclanthology.org/2023.findings-acl.355/",
    doi = "10.18653/v1/2023.findings-acl.355",
    pages = "5747--5758"
}

@inproceedings{igamberdiev-habernal-2023-dp,
    title = "{DP}-{BART} for Privatized Text Rewriting under Local Differential Privacy",
    author = "Igamberdiev, Timour  and
      Habernal, Ivan",
    editor = "Rogers, Anna  and
      Boyd-Graber, Jordan  and
      Okazaki, Naoaki",
    booktitle = "Findings of the Association for Computational Linguistics: ACL 2023",
    month = jul,
    year = "2023",
    address = "Toronto, Canada",
    publisher = "Association for Computational Linguistics",
    url = "https://aclanthology.org/2023.findings-acl.874/",
    doi = "10.18653/v1/2023.findings-acl.874",
    pages = "13914--13934"
}

@inproceedings{liu2024purpose,
  title={Being Transparent is Merely the Beginning: Enforcing Purpose Limitation with Polynomial Approximation},
  author={Liu, Shuofeng and Wang, Zihan and Xue, Minhui and Wang, Long and Zhang, Yuanchao and Bai, Guangdong.},
  journal={USENIX Security},
  year={2024}
}

@inproceedings{ma2023formalizing,
  title={Formalizing Robustness Against Character-Level Perturbations for Neural Network Language Models},
  author={Ma, Zhongkui and Feng, Xinguo and Wang, Zihan and Liu, Shuofeng and Ma, Mengyao and Guan, Hao and Meng, Mark Huasong},
  booktitle={International Conference on Formal Engineering Methods},
  pages={100--117},
  year={2023},
  organization={Springer}
}

@inproceedings{
Holtzman2020The,
title={The Curious Case of Neural Text Degeneration},
author={Ari Holtzman and Jan Buys and Li Du and Maxwell Forbes and Yejin Choi},
booktitle={International Conference on Learning Representations},
year={2020},
url={https://openreview.net/forum?id=rygGQyrFvH}
}

@article{pape2024prompt,
  title={Prompt obfuscation for large language models},
  author={Pape, David and Mavali, Sina and Eisenhofer, Thorsten and Sch{\"o}nherr, Lea},
  journal={arXiv preprint arXiv:2409.11026},
  year={2024}
}

@inproceedings{le2007continuous,
  title={Continuous neural networks},
  author={Le Roux, Nicolas and Bengio, Yoshua},
  booktitle={Artificial Intelligence and Statistics},
  pages={404--411},
  year={2007},
  organization={PMLR}
}

@inproceedings{radiya2020fine,
  title={How fine can fine-tuning be? learning efficient language models},
  author={Radiya-Dixit, Evani and Wang, Xin},
  booktitle={International Conference on Artificial Intelligence and Statistics},
  pages={2435--2443},
  year={2020},
  organization={PMLR}
}

@article{zhang2018efficient,
  title={Efficient neural network robustness certification with general activation functions},
  author={Zhang, Huan and Weng, Tsui-Wei and Chen, Pin-Yu and Hsieh, Cho-Jui and Daniel, Luca},
  journal={Advances in neural information processing systems},
  volume={31},
  year={2018}
}

@article{ye2024fedllm,
  title={Fedllm-bench: Realistic benchmarks for federated learning of large language models},
  author={Ye, Rui and Ge, Rui and Zhu, Xinyu and Chai, Jingyi and Yaxin, Du and Liu, Yang and Wang, Yanfeng and Chen, Siheng},
  journal={Advances in Neural Information Processing Systems},
  volume={37},
  pages={111106--111130},
  year={2024}
}

@inproceedings{zhang2023agrevader,
  title={AgrEvader: Poisoning membership inference against Byzantine-robust federated learning},
  author={Zhang, Yanjun and Bai, Guangdong and Chamikara, Mahawaga Arachchige Pathum and Ma, Mengyao and Shen, Liyue and Wang, Jingwei and Nepal, Surya and Xue, Minhui and Wang, Long and Liu, Joseph},
  booktitle={Proceedings of the ACM Web Conference 2023},
  pages={2371--2382},
  year={2023}
}

@inproceedings{ma2024unveiling,
  title={Unveiling intellectual property vulnerabilities of gan-based distributed machine learning through model extraction attacks},
  author={Ma, Mengyao and Liu, Shuofeng and Chamikara, MAP and Baruwal Chhetri, Mohan and Bai, Guangdong},
  booktitle={Proceedings of the 33rd ACM International Conference on Information and Knowledge Management},
  pages={1617--1626},
  year={2024}
}

@inproceedings{vo2025practical,
  title={Practical Poisoning Attacks with Limited Byzantine Clients in Clustered Federated Learning},
  author={Vo, Viet and Ma, Mengyao and Bai, Guangdong and Ko, Ryan and Neplal, Surya},
  booktitle={2025 IEEE Symposium on Security and Privacy (SP)},
  pages={1751--1769},
  year={2025},
  organization={IEEE}
}

@inproceedings{wang2025ai,
  title={AI Model Modulation with Logits Redistribution},
  author={Wang, Zihan and Ma, Zhongkui and Feng, Xinguo and Mei, Zhiyang and Ma, Ethan and Wang, Derui and Xue, Minhui and Bai, Guangdong},
  booktitle={Proceedings of the ACM on Web Conference 2025},
  pages={4699--4709},
  year={2025}
}

@inproceedings{wang2024corelocker,
  title={Corelocker: Neuron-level usage control},
  author={Wang, Zihan and Ma, Zhongkui and Feng, Xinguo and Sun, Ruoxi and Wang, Hu and Xue, Minhui and Bai, Guangdong},
  booktitle={2024 IEEE Symposium on Security and Privacy (SP)},
  pages={2497--2514},
  year={2024},
  organization={IEEE}
}

@article{wang2025re,
  title={Re-Key-Free, Risky-Free: Adaptable Model Usage Control},
  author={Wang, Zihan and Ma, Zhongkui and Feng, Xinguo and Yan, Chuan and Liu, Dongge and Sun, Ruoxi and Wang, Derui and Xue, Minhui and Bai, Guangdong},
  journal={arXiv preprint arXiv:2511.18772},
  year={2025}
}

@article{wang2025catch,
  title={Catch-Only-One: Non-Transferable Examples for Model-Specific Authorization},
  author={Wang, Zihan and Ma, Zhiyong and Ma, Zhongkui and Liu, Shuofeng and Liu, Akide and Wang, Derui and Xue, Minhui and Bai, Guangdong},
  journal={arXiv preprint arXiv:2510.10982},
  year={2025}
}

\appendix
\section{Additional Experimental Settings}

\subsection{Hyperparameters Study}
\label{appendix:hyper}
We run a grid search with the GPT-2 model and the Enron Emails dataset for the following hyperparameters of \codename: top-$k$ (10, 20, $\cdots$, 100), overlap threshold $\tau_o$ (0.1, 0.2, $\cdots$, 0.5), and beam search width (1, 2, $\cdots$, 5). Our results show that \codename is robust across these hyperparameter combinations. We select a configuration of top-$k = 70$, $\tau_o = 0.1$, and beams $= 1$, as this setup achieves a balanced trade-off among utility, privacy, and computational efficiency. The complete parameter search results are in Table~\ref{table:param_search}.

\begin{table}[!htp]

\renewcommand{\arraystretch}{.6}
\setlength{\tabcolsep}{3pt}

\centering
\caption{Hyperparamters study.}
\label{table:param_search}

\begin{threeparttable}

\begin{tabular}{
cc
rr p{0pt}
rrrr
}

\toprule
&&
\multirow{3}{*}{\makecell[c]{Loss}} &
\multirow{3}{*}{\makecell[c]{PPL}} &&
\multicolumn{4}{c}{Text Similarity} \\
\cmidrule(l{1pt}r{1pt}){6-9}

&&&&&
\makecell[c]{R-1} & \makecell[c]{R-2} & \makecell[c]{R-L} & \makecell[c]{M} \\
\midrule

Pre-trained & \makecell[c]{---} & 8.17 & 3540.01 && \makecell[c]{---} & \makecell[c]{---} & \makecell[c]{---} & \makecell[c]{---} \\
\midrule

Original & \makecell[c]{---} & 2.4 & 11.05 && \makecell[c]{---} & \makecell[c]{---} & \makecell[c]{---} & \makecell[c]{---} \\
\midrule

\multirow{10}{*}{\makecell[c]{Obfuscated\\(Beam=1,\\$\tau_o$=0.1)}}
& Top-$k$=10  & 2.89 & 18.04 && 0.12 & 0.01 & 0.11 & 0.12 \\
& Top-$k$=20  & 2.93 & 18.68 && 0.11 & 0.01 & 0.10 & 0.12 \\
& Top-$k$=30  & 2.92 & 18.61 && 0.09 & 0.01 & 0.08 & 0.10 \\
& Top-$k$=40  & 2.94 & 19.00 && 0.08 & 0.00 & 0.08 & 0.10 \\
& Top-$k$=50  & 2.95 & 19.17 && 0.07 & 0.00 & 0.06 & 0.07 \\
& Top-$k$=60  & 2.95 & 19.16 && 0.06 & 0.00 & 0.06 & 0.08 \\
& Top-$k$=\textbf{70} & \textbf{2.95} & \textbf{19.18} &&  \textbf{0.04} & \textbf{0.00} & \textbf{0.04} & \textbf{0.05} \\
& Top-$k$=80  & 2.95 & 19.10 && 0.03 & 0.00 & 0.03 & 0.05 \\
& Top-$k$=90  & 2.96 & 19.24 && 0.03 & 0.00 & 0.03 & 0.06 \\
& Top-$k$=100 & 2.96 & 19.21 && 0.08 & 0.00 & 0.08 & 0.08 \\

\midrule

\multirow{5}{*}{\makecell[c]{Obfuscated\\(Top-$k$=70,\\$\tau_o$=0.1)}}
& Beam=\textbf{1} & \textbf{2.95} & \textbf{19.98} && \textbf{0.04} & \textbf{0.00} & \textbf{0.04} & \textbf{0.05} \\
& Beam=2 & 2.95 & 19.98 && 0.04 & 0.00 & 0.04 & 0.05 \\
& Beam=3 & 2.95 & 19.98 && 0.04 & 0.00 & 0.04 & 0.05 \\
& Beam=4 & 2.96 & 19.31 && 0.04 & 0.00 & 0.04 & 0.05 \\
& Beam=5 & 2.95 & 19.16 && 0.04 & 0.00 & 0.04 & 0.05 \\

\midrule

\multirow{5}{*}{\makecell[c]{Obfuscated\\(Top-$k$=70,\\Beam=1)}}
& $\tau_o$=\textbf{0.1} & \textbf{2.95} & \textbf{19.18} && \textbf{0.04} & \textbf{0.00} & \textbf{0.04} & \textbf{0.05} \\
& $\tau_o$=0.2 & 2.94 & 18.88 && 0.04 & 0.00 & 0.04 & 0.05 \\
& $\tau_o$=0.3 & 2.94 & 18.84 && 0.05 & 0.00 & 0.05 & 0.06 \\
& $\tau_o$=0.4 & 2.94 & 18.84 && 0.06 & 0.00 & 0.05 & 0.06 \\
& $\tau_o$=0.5 & 2.93 & 18.76 && 0.06 & 0.00 & 0.05 & 0.06 \\

\bottomrule
\end{tabular}

\end{threeparttable}
\end{table}

\subsection{LLM Generation Adaptive Attack}
\label{appendix:LLM_gen_attack}
\myparagraph{Model}
We use the advanced Llama-3.2-3B as the base model for fine-tuning the LLM attack model for all target models.

\myparagraph{Datasets}
In the LLM generation adaptive attack, Tweet Sentiment Analysis and Yahoo Answers Topics are paired as reference datasets due to their similar topics and shared domains. For the remaining datasets, the following target-reference pairs are employed: SST-2 with Rotten Tomatoes~\cite{rotten_tomatoes_dataset}, Enron Emails with Fine Personas Emails Conversation Summaries~\cite{fine_persona}, Open Australian Legal Corpus with Legal Task~\cite{legal_task}, and AG News with News Dataset Seq~\cite{news_dataset}. These pairings are thoughtfully chosen based on their thematic and domain relevance, providing a contextual similarity for training the attack model.
\section{Human Evaluations}
We detail the human evaluations on semantic similarity and generative model utility below.

\subsection{Human Evaluation on Semantic Similarity}
\label{appendix:human_eval_semantics}

In the main experiments, we use ROUGE (exact match) and METEOR (partial semantic match) scores to evaluate the similarity between original and obfuscated data.
This unsurprisingly achieves low similarity, as shadow tokens and original tokens differ in text, and the semantics overlap is limited to stems and synonyms.
To the best of our knowledge, no existing metric fully captures semantic similarity between original and obfuscated data.
To address this, we conduct a human evaluation to gain a comprehensive understanding and bridge the evaluation gap.

\myparagraph{Data selection} We divide all obfuscated data into 10 groups based on their R-1 scores with the original data, using intervals (\textit{e.g.}, 0.0--0.1, 0.2--0.3, $\dots$, 0.9--1.0). We randomly select 50 sentence pairs (an original and its obfuscated version) from these groups for evaluation, maintaining the sample distribution ratios. The sampled data is further screened to ensure no harmful content.

\myparagraph{Evaluation} Ethical approval for human-subject research is obtained for this evaluation from our institute. Ten external volunteers serve as evaluators and independently provide their evaluations. Evaluators receive instructions to rate the semantic similarity between each original and obfuscated sentence pair on a scale of 1 to 5, where higher scores indicate greater similarity. 
This structured approach ensures consistency in human evaluation. The detailed instruction is provided below.

\begin{tcolorbox}[
    colback=white,
    fontupper=\footnotesize,
    colframe=skyblue!40!white, 
    width=\columnwidth,       
    sharp corners,           
    title=Human evaluation instructions for semantic similarity., 
    fonttitle=\bfseries,     
    coltitle=black,           
    left=0mm, right=0mm, top=1mm, bottom=1mm,
    breakable
]
On a scale of 1 to 5 (1 being the lowest and 5 being the highest), rate the semantic similarity between the two sentences:

1: The sentences have completely different meanings, with no discernible overlap. The second sentence may be nonsensical or entirely unrelated.

2: The sentences share minimal overlap, but the meanings are mostly different or unclear. Key ideas are missing or distorted.

3: The sentences share some meaning but differ significantly in phrasing, content, or clarity. The overlap is partial but noticeable.

4: The sentences are mostly similar in meaning but may have slight differences in wording or emphasis. Key ideas are largely preserved.

5: The sentences convey almost identical meanings, with minimal or no differences in content or phrasing.
\end{tcolorbox}

 \myparagraph{Results} We report the evaluation results. The ratings range from 1 to 5, with a mean of 1.43, indicating that the sentence pairs exhibit minimal semantic similarity. The rating variances range from 0.0 to 1.6, with a mean of 0.37, indicating low inter-rater variability and consistent evaluator agreement. These findings validate \codename's effectiveness in disrupting semantics, demonstrating its strength as a data privacy mechanism.

 \subsection{Human Evaluation on Generative Model Utility}
\label{appendix:human_eval_generation}
In Section~\ref{gen_model_utility}, we evaluate the utility preservation of \codename by comparing the losses and perplexities of the original and defended models on test data. To further assess the similarity between their generations, we report ROUGE and METEOR scores, which capture lexical and partial semantic overlap. However, these metrics are limited in reflecting fluency and topic relevance, which are key aspects of practical generative utility. To address this limitation, we conduct a human evaluation to provide a more comprehensive assessment of the generation quality from both the original and defended models.

\myparagraph{Data selection} The first quarter of each test sentence from the three datasets for generative models is used to prompt the models. To encourage more diverse outputs, Nucleus sampling (top-p sampling)~\cite{Holtzman2020The} is applied with a top-p of 0.95 and a temperature of 0.8. We use a maximum length of 128 tokens, so that the generation length is not constrained by the original sentence length. The generations from both models are shuffled, and 50 generated samples are randomly selected for the human evaluation. The sampled data is further screened to ensure no harmful content.

\myparagraph{Evaluation} Ethical approval for human-subject research is obtained for this evaluation from our institute. Ten external volunteers serve as evaluators and independently provide their evaluations. Evaluators are instructed to rate each generation's fluency and topic relevance on a scale from 1 to 5, with higher scores indicating better quality. To prevent potential bias in topic relevance assessment, the authors conducting the evaluation are blinded to the source dataset of each sample. 
The detailed evaluation instructions are provided below.

\begin{tcolorbox}[
    colback=white,
    fontupper=\footnotesize,
    colframe=skyblue!40!white, 
    width=\columnwidth,       
    sharp corners,           
    title=Human evaluation instructions for generation quality., 
    fonttitle=\bfseries,     
    coltitle=black,           
    left=0mm, right=0mm, top=1mm, bottom=1mm,
    breakable
]
On a scale of 1 to 5 (1 being the lowest and 5 being the highest), rate the fluency and topic relevance to its category for each text. Note that the text is restricted to a certain length. It's normal for a long text to end in the middle of a sentence if it exceeds the length. Do not penalize scores based on this.

Fluency:
Rate how natural, clear, and grammatically correct the text is.

1: The text is hard to read, ungrammatical, or nonsensical. It may have serious errors or be incoherent.

2: The text has multiple issues with grammar, structure, or word choice, making it awkward or unclear.

3: The text is somewhat readable but contains noticeable errors or awkward phrasing. Overall meaning is understandable.

4: The text reads well with minor issues. It’s mostly clear and well-formed.

5: The text is smooth, grammatically correct, and natural-sounding with no major issues.

Topic relevance:

Rate how well the text fits its intended category or subject.

1: The text is completely off-topic or irrelevant to the category.

2: The text has minimal connection to the topic. Key ideas are missing or unrelated.

3: The text is partially on-topic but may stray, lack focus, or contain unrelated content.

4: The text mostly aligns with the topic, with only minor deviations or missing details.

5: The text is clearly relevant and strongly aligned with the topic. All key ideas are appropriately addressed.

\end{tcolorbox}

\myparagraph{Results} We report the human evaluation results on fluency and topic relevance of model generations. For fluency, the original models receive ratings between 1 and 5, with a mean of 3.55 and variances ranging from 0 to 2.18 (mean variance: 1.20). The defended models are rated from 1 to 5, with a mean of 3.86 and variances ranging from 0.49 to 2.54 (mean: 1.28). For topic relevance, the original models receive scores from 1 to 5, with a mean of 3.65 and variances ranging from 0 to 2.49 (mean: 1.23). The defended models also receive ratings from 1 to 5, with a mean of 3.88 and variances between 0.18 and 4.01 (mean: 1.41). These consistently moderate to high mean ratings, alongside relatively low mean variances, suggest that both the original and defended models generate highly fluent and topically relevant text, demonstrating the utility-preserving capability of \codename on generative models.
\section{Supplementary Figures and Tables}

\subsection{Adaptive Attacks Results}
\label{appendix: adaptive}
Table~\ref{table:adaptive_attack} presents the results of the five proposed adaptive attacks. Despite improved attack efficacy in most adaptive attacks, overall data recovery remains limited, with R-1 scores ranging from 0.00 to 0.29 and METEOR scores from 0.00 to 0.25. Considering the near-perfect recovery achieved by DAGER, \codename demonstrates exceptional robustness, successfully resisting even the most challenging adaptive attacks with full adversarial knowledge of defense methods and parameters.

\subsection{Ablation study}
\label{appendix: ablation}
Table~\ref{table:ablation} presents the ablation study. Disabling one or all similarity heuristics increases semantic similarity (METEOR up to 0.2), confirming that heuristics jointly reduce overlap and enhance privacy. For selection, random search yields poor utility, no search achieves moderate utility, and our optimization method attains the highest utility, demonstrating its effectiveness in preserving performance under obfuscation.

\subsection{Unsupported Experiments}
\label{appendix:unsupported}
We summarize the supported and unsupported models and defenses for all attack methods mentioned in Section~\ref{sec:defense_efficacy}.

\myparagraph{Models} For classification models, DLG, TAG, and LAMP only support attacking models from the BERT family, GRAB supports all models, and DAGER only supports attacking BERT-base-uncased. For generative models, the authors of DAGER provide modified versions of TAG and LAMP to support attacks on GPT-2, while DLG and GRAB do not support generative models. However, upon further investigation of the implementation, these modified versions of TAG and LAMP use GPT-2 as a classification model, which is not a standard approach for generative tasks. DAGER only supports GPT-2, Llama-2-7b, Llama-3-8b, and Llama-3.1-8b.

\myparagraph{Defenses} Most attacks support all evaluated defenses. The implementation of DAGER is currently encountering an error with the gradient noise defense.
This has been raised to the authors, but there is no response from them at the time of writing.

\begin{table}[!htp]

\renewcommand{\arraystretch}{.6}
\setlength{\tabcolsep}{3pt}

\centering
\caption{Ablation study of BERT-base-uncased on SST-2.}
\label{table:ablation}

\begin{threeparttable}

\begin{tabular}{
ccc p{0pt}
ccc p{0pt}
cc p{0pt}
cccc
}
\toprule
\multicolumn{3}{c}{Similarity} &&
\multicolumn{3}{c}{Search} &&
\multicolumn{2}{c}{Utility} &&
\multicolumn{4}{c}{Privacy}\\
\cmidrule(l{1pt}r{1pt}){1-3}
\cmidrule(l{1pt}r{1pt}){5-7}
\cmidrule(l{1pt}r{1pt}){9-10}
\cmidrule(l{1pt}r{1pt}){12-15}
\multicolumn{1}{c}{IS} & \multicolumn{1}{c}{DS} &\multicolumn{1}{c}{CS} &&
\multicolumn{1}{c}{RS} & \multicolumn{1}{c}{NS} & \multicolumn{1}{c}{OP} &&
\multicolumn{1}{c}{Acc.} & \multicolumn{1}{c}{F1} &&
\multicolumn{1}{c}{R-1} & \multicolumn{1}{c}{R-2} & \multicolumn{1}{c}{R-L} & \multicolumn{1}{c}{M}\\
\midrule
\ding{51}&\ding{51}&\ding{51}  && \ding{51} && && 0.63 & 0.73 && 0.01 & 0.00 & 0.01 & 0.01 \\
\ding{51}&\ding{51}&\ding{51}  &&  &\ding{51}& && 0.72 & 0.78 && 0.02 & 0.00 & 0.02 & 0.03 \\
\ding{51}&\ding{51}&\ding{51}  &&  &&\ding{51} && \textbf{0.82} & \textbf{0.85} && 0.02 & 0.00 & 0.02 & 0.03 \\

& \ding{51} & \ding{51} && \ding{51} && && 0.56 & 0.70 && 0.00 & 0.00 & 0.00 & 0.00 \\
& \ding{51} & \ding{51} &&  &\ding{51}& && 0.71 & 0.78 && 0.03 & 0.00 & 0.03 & 0.04 \\
& \ding{51} & \ding{51} &&  &&\ding{51} && \textbf{0.80} & \textbf{0.82} && 0.01 & 0.00 & 0.01 & 0.02 \\

\ding{51} && \ding{51} && \ding{51} && && 0.56 & 0.70 && 0.01 & 0.00 & 0.01 & 0.02 \\
\ding{51} && \ding{51} && &\ding{51}& && 0.70 & 0.76 && 0.02 & 0.00 & 0.02 & 0.05 \\
\ding{51} && \ding{51} && &&\ding{51} && \textbf{0.82} & \textbf{0.83} && 0.03 & 0.00 & 0.02 & 0.03 \\

\ding{51} &\ding{51}&  && \ding{51} && && 0.58 & 0.63 && 0.01 & 0.00 & 0.01 & 0.02 \\
\ding{51} &\ding{51}&  && &\ding{51}& && 0.78 & 0.81 && 0.02 & 0.00 & 0.02 & 0.05 \\
\ding{51} &\ding{51}&  && &&\ding{51} && \textbf{0.86} & \textbf{0.87} && 0.02 & 0.00 & 0.02 & 0.04 \\

\ding{51}&&  && \ding{51} && && 0.66 & 0.74 && 0.01 & 0.00 & 0.01 & 0.02 \\
\ding{51}&&  &&  &\ding{51}& && 0.74 & 0.79 && 0.02 & 0.00 & 0.02 & \textbf{0.10} \\
\ding{51}&&  &&  &&\ding{51} && \textbf{0.85} & \textbf{0.87} && 0.03 & 0.00 & 0.02 & 0.06 \\

&\ding{51}&  && \ding{51} && && 0.59 & 0.71 && 0.00 & 0.00 & 0.00 & 0.01 \\
&\ding{51}&  &&  &\ding{51}& && 0.84 & 0.86 && 0.03 & 0.00 & 0.03 & 0.06 \\
&\ding{51}&  &&  &&\ding{51} && \textbf{0.88} & \textbf{0.89} && 0.01 & 0.00 & 0.01 & 0.03 \\

&&\ding{51}  && \ding{51} && && 0.57 & 0.71 && 0.01 & 0.00 & 0.01 & 0.01 \\
&&\ding{51}  &&  &\ding{51}& && 0.82 & 0.84 && 0.03 & 0.00 & 0.02 & 0.07 \\
&&\ding{51}  &&  &&\ding{51} && \textbf{0.84} & \textbf{0.85} && 0.04 & 0.00 & 0.03 & 0.05 \\

&&  && \ding{51} && && 0.59 & 0.72 && 0.01 & 0.00 & 0.01 & 0.01 \\
&&  && &\ding{51}& && 0.88 & 0.88 && 0.02 & 0.00 & 0.02 & \textbf{0.20} \\
&&  && &&\ding{51} && \textbf{0.87} & \textbf{0.88} && 0.04 & 0.00 & 0.03 & \textbf{0.10} \\

\bottomrule
\end{tabular}

\begin{tablenotes}
    \footnotesize
    \item IS: Indirect similarity; DS: Direct similarity; CS: Common lemma similarity; RS: Random search; NS: No search; OP: Optimization.
\end{tablenotes}

\end{threeparttable}
\end{table}

\myparagraph{Methods} We also highlight that the attack method FILM~\cite{film} for generative models is not evaluated. Upon further investigation of its implementation, it is not a standard GIA method that recovers one sentence for a sample. It aims to recover multiple candidates and select the optimal one. However, the selection process it unclear.

\subsection{Validation of Theoretical Analysis}
\label{appendix: validation_theory}
Figure~\ref{fig:epsilons} shows that both bounds are concentrated in the low range and that $\epsilon_d$ largely determines $\epsilon$, as indicated by the 1-slope diagonal trends. Figure~\ref{fig:loss_vs_grad} shows the pair-wise scatter with linear regression (log-scaled vertical axis). Loss deviations remain consistently small, while gradient deviations are much larger and grow at a substantially higher fitted slope, confirming our theoretical prediction that utility is preserved but inversion is strongly suppressed.

\begin{figure}[!htbp]
    \centering
    \includegraphics[width=0.9\linewidth]{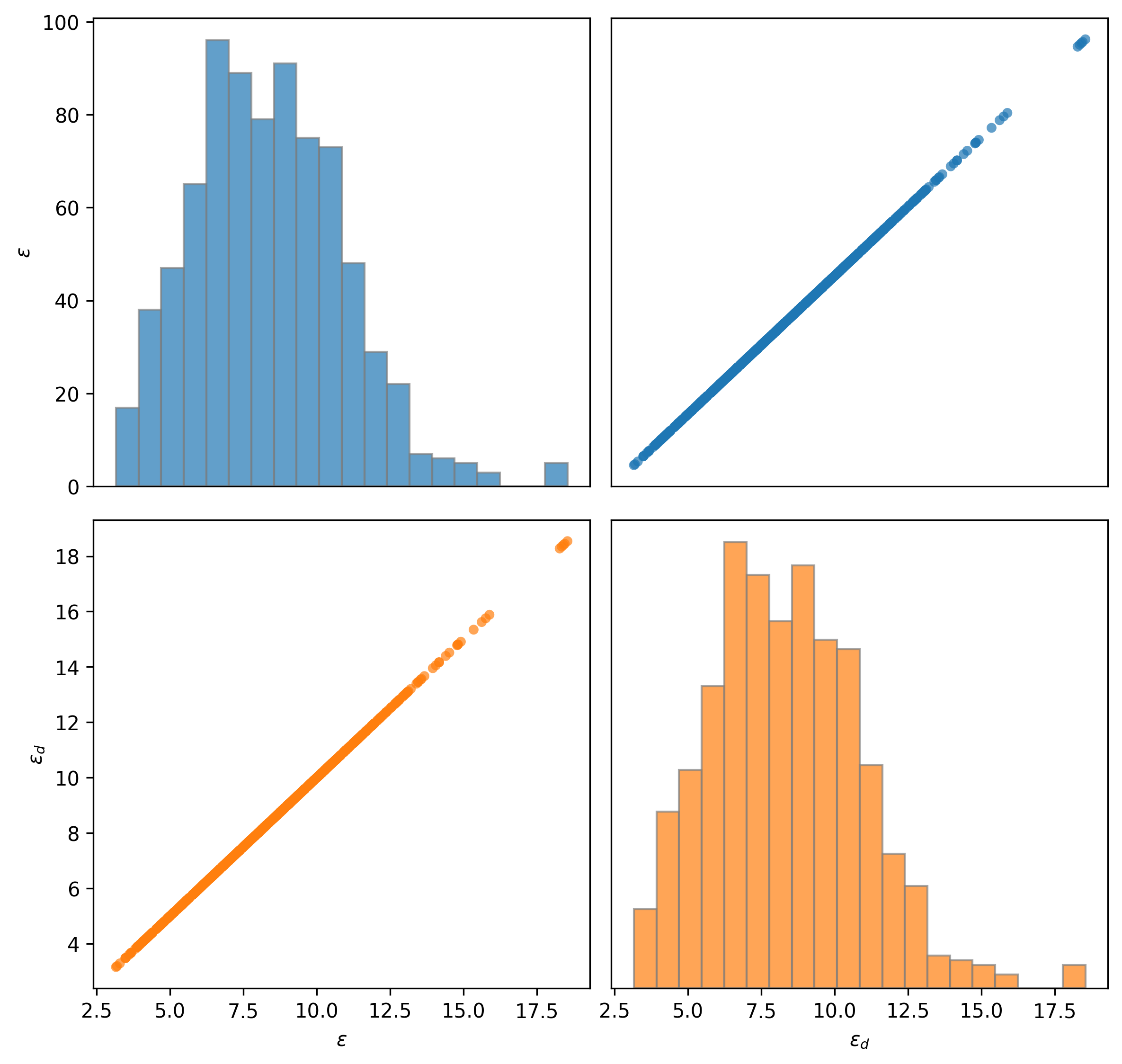}
    \caption{Upper Bounds of $\epsilon$.}
    \label{fig:epsilons}
\end{figure}
\begin{figure}[!htbp]
    \centering
    \includegraphics[width=0.9\linewidth]{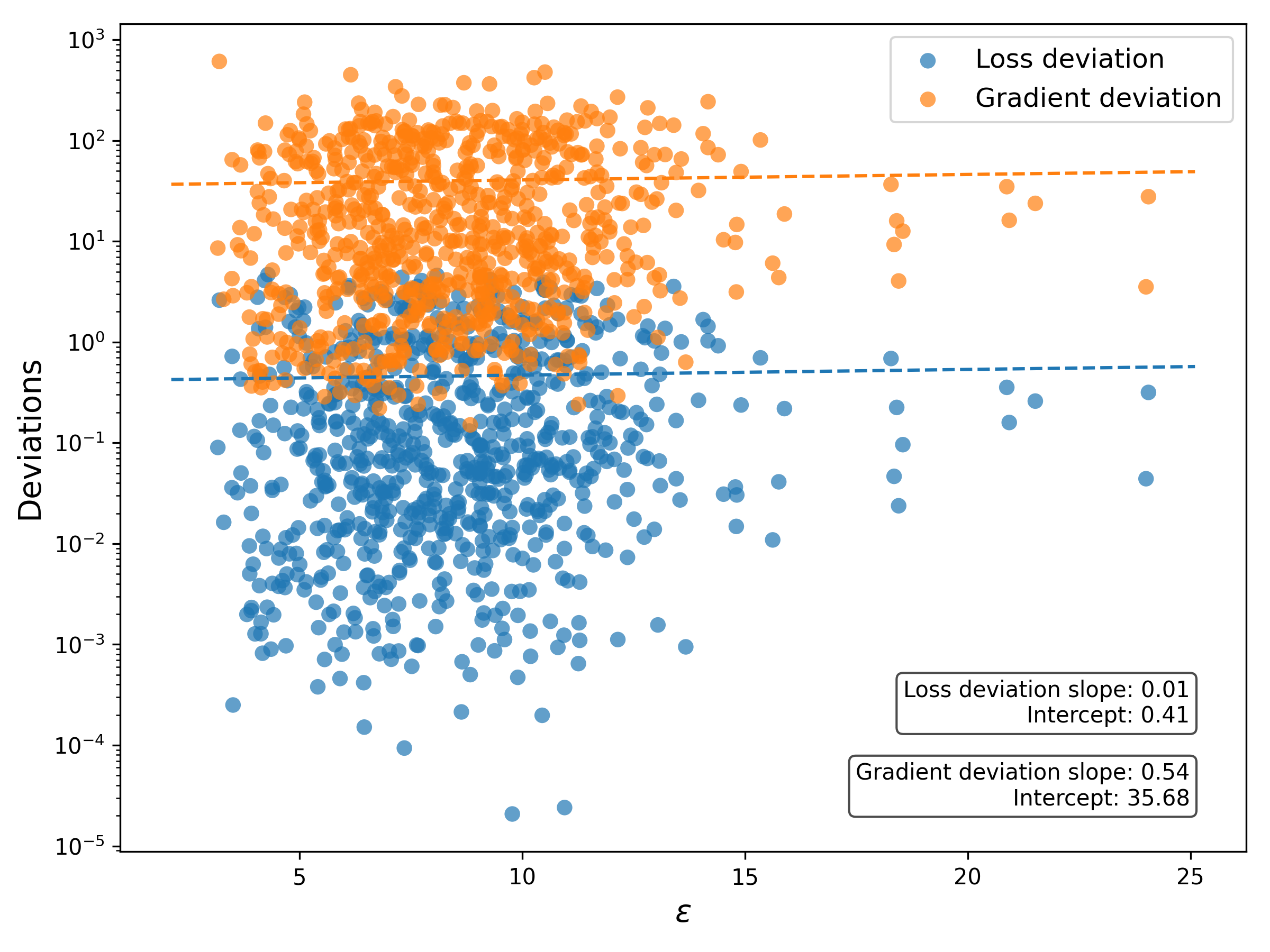}
    \caption{Loss and gradient deviation comparison.}
    \label{fig:loss_vs_grad}
\end{figure}

\afterpage{
    \begin{table*}[!htbp]

\renewcommand{\arraystretch}{.5}
\setlength{\tabcolsep}{2pt}

\centering
\caption{Attack efficacy of adaptive attacks.}
\label{table:adaptive_attack}

\begin{threeparttable}

\resizebox{!}{0.49\textheight}{

\begin{tabular}{
l
rrrr p{1pt}
rrrr p{1pt}
rrrr p{1pt}
rrrr p{1pt}
rrrr
}
\toprule
& 
\multicolumn{4}{c}{Sampling} &&
\multicolumn{4}{c}{Max Similarity} &&
\multicolumn{4}{c}{Median Similarity} &&
\multicolumn{4}{c}{Mean Similarity} &&
\multicolumn{4}{c}{LLM generation} \\
\cmidrule(l{1pt}r{1pt}){2-5}
\cmidrule(l{1pt}r{1pt}){7-10} 
\cmidrule(l{1pt}r{1pt}){12-15}
\cmidrule(l{1pt}r{1pt}){17-20}
\cmidrule(l{1pt}r{1pt}){22-25}
& 
\multicolumn{1}{c}{R-1} & \multicolumn{1}{c}{R-2} & \multicolumn{1}{c}{R-L} & \multicolumn{1}{c}{M} &&
\multicolumn{1}{c}{R-1} & \multicolumn{1}{c}{R-2} & \multicolumn{1}{c}{R-L} & \multicolumn{1}{c}{M} &&
\multicolumn{1}{c}{R-1} & \multicolumn{1}{c}{R-2} & \multicolumn{1}{c}{R-L} & \multicolumn{1}{c}{M} &&
\multicolumn{1}{c}{R-1} & \multicolumn{1}{c}{R-2} & \multicolumn{1}{c}{R-L} & \multicolumn{1}{c}{M} &&
\multicolumn{1}{c}{R-1} & \multicolumn{1}{c}{R-2} & \multicolumn{1}{c}{R-L} & \multicolumn{1}{c}{M} \\

\midrule
& 
\multicolumn{24}{c}{{SST2}} \\
\midrule

BERT-base-uncased &
0.20 & 0.05 & 0.20 & 0.14 &&
0.20 & 0.04 & 0.20 & 0.15 &&
0.20 & 0.03 & 0.19 & 0.14 &&
0.21 & 0.04 & 0.21 & 0.17 &&
0.10 & 0.01 & 0.09 & 0.08 \\

BERT-base-cased &
0.09 & 0.01 & 0.08 & 0.06 &&
0.07 & 0.00 & 0.07 & 0.06 &&
0.07 & 0.00 & 0.07 & 0.05 &&
0.09 & 0.01 & 0.08 & 0.05 &&
0.07 & 0.01 & 0.06 & 0.05 \\

BERT-large-uncased &
0.12 & 0.02 & 0.12 & 0.08 &&
0.15 & 0.02 & 0.14 & 0.11 &&
0.11 & 0.01 & 0.10 & 0.06 &&
0.15 & 0.02 & 0.15 & 0.11 &&
0.08 & 0.01 & 0.07 & 0.07 \\

BERT-large-cased &
0.08 & 0.01 & 0.07 & 0.05 &&
0.14 & 0.02 & 0.13 & 0.09 &&
0.06 & 0.00 & 0.06 & 0.04 &&
0.10 & 0.01 & 0.09 & 0.07 &&
0.10 & 0.01 & 0.08 & 0.08 \\

RoBERTa-base &
0.05 & 0.00 & 0.05 & 0.03 &&
0.05 & 0.00 & 0.05 & 0.03 &&
0.05 & 0.00 & 0.04 & 0.03 &&
0.04 & 0.00 & 0.04 & 0.03 &&
0.09 & 0.01 & 0.07 & 0.07 \\

RoBERTa-large &
0.04 & 0.00 & 0.04 & 0.03 &&
0.07 & 0.00 & 0.06 & 0.04 &&
0.05 & 0.00 & 0.05 & 0.04 &&
0.07 & 0.00 & 0.06 & 0.04 &&
0.09 & 0.01 & 0.08 & 0.08 \\

DeBERTa-v3-small &
0.07 & 0.00 & 0.07 & 0.04 &&
0.11 & 0.01 & 0.11 & 0.07 &&
0.07 & 0.00 & 0.06 & 0.05 &&
0.14 & 0.02 & 0.14 & 0.10 &&
0.11 & 0.01 & 0.10 & 0.09 \\

DeBERTa-v3-base &
0.01 & 0.00 & 0.01 & 0.01 &&
0.10 & 0.01 & 0.10 & 0.07 &&
0.01 & 0.00 & 0.01 & 0.00 &&
0.09 & 0.00 & 0.08 & 0.06 &&
0.06 & 0.00 & 0.06 & 0.05 \\

DeBERTa-v3-large &
0.05 & 0.00 & 0.05 & 0.04 &&
0.17 & 0.02 & 0.16 & 0.14 &&
0.04 & 0.00 & 0.04 & 0.03 &&
0.18 & 0.03 & 0.18 & 0.15 &&
0.18 & 0.04 & 0.17 & 0.16 \\
\midrule

&
\multicolumn{24}{c}{{Tweet}} \\
\midrule

BERT-base-uncased &
0.27 & 0.08 & 0.27 & 0.22 &&
0.25 & 0.07 & 0.25 & 0.20 &&
0.26 & 0.07 & 0.26 & 0.20 &&
0.25 & 0.06 & 0.24 & 0.21 &&
0.17 & 0.03 & 0.15 & 0.12 \\

BERT-base-cased &
0.09 & 0.01 & 0.08 & 0.07 &&
0.06 & 0.00 & 0.06 & 0.06 &&
0.08 & 0.01 & 0.08 & 0.06 &&
0.09 & 0.01 & 0.09 & 0.06 &&
0.09 & 0.01 & 0.08 & 0.07 \\

BERT-large-uncased &
0.15 & 0.02 & 0.15 & 0.11 &&
0.17 & 0.03 & 0.16 & 0.12 &&
0.15 & 0.02 & 0.14 & 0.10 &&
0.17 & 0.03 & 0.17 & 0.12 &&
0.10 & 0.01 & 0.09 & 0.08 \\

BERT-large-cased &
0.09 & 0.01 & 0.08 & 0.07 &&
0.12 & 0.01 & 0.11 & 0.09 &&
0.10 & 0.01 & 0.09 & 0.06 &&
0.10 & 0.01 & 0.09 & 0.06 &&
0.11 & 0.01 & 0.09 & 0.08 \\

RoBERTa-base &
0.08 & 0.01 & 0.08 & 0.06 &&
0.05 & 0.00 & 0.05 & 0.05 &&
0.06 & 0.00 & 0.06 & 0.04 &&
0.04 & 0.00 & 0.04 & 0.04 &&
0.11 & 0.01 & 0.09 & 0.08 \\

RoBERTa-large &
0.06 & 0.00 & 0.06 & 0.05 &&
0.08 & 0.01 & 0.08 & 0.07 &&
0.07 & 0.00 & 0.07 & 0.04 &&
0.08 & 0.00 & 0.08 & 0.06 &&
0.11 & 0.01 & 0.09 & 0.07 \\

DeBERTa-v3-small &
0.07 & 0.00 & 0.07 & 0.05 &&
0.10 & 0.01 & 0.10 & 0.06 &&
0.06 & 0.00 & 0.06 & 0.04 &&
0.14 & 0.02 & 0.13 & 0.11 &&
0.13 & 0.01 & 0.11 & 0.10 \\

DeBERTa-v3-base &
0.01 & 0.00 & 0.01 & 0.02 &&
0.09 & 0.01 & 0.09 & 0.07 &&
0.00 & 0.00 & 0.00 & 0.01 &&
0.11 & 0.01 & 0.10 & 0.08 &&
0.07 & 0.00 & 0.06 & 0.05 \\

DeBERTa-v3-large &
0.07 & 0.00 & 0.07 & 0.05 &&
0.20 & 0.04 & 0.19 & 0.17 &&
0.06 & 0.00 & 0.06 & 0.04 &&
0.19 & 0.03 & 0.19 & 0.16 &&
0.24 & 0.07 & 0.22 & 0.19 \\
\midrule

& 
\multicolumn{24}{c}{{Yahoo}} \\
\midrule

BERT-base-uncased &
0.28 & 0.08 & 0.27 & 0.22 &&
0.25 & 0.08 & 0.24 & 0.19 &&
0.28 & 0.09 & 0.28 & 0.20 &&
0.29 & 0.10 & 0.29 & 0.24 &&
0.21 & 0.05 & 0.19 & 0.18 \\

BERT-base-cased &
0.09 & 0.00 & 0.09 & 0.08 &&
0.07 & 0.00 & 0.07 & 0.05 &&
0.08 & 0.01 & 0.08 & 0.08 &&
0.11 & 0.01 & 0.10 & 0.09 &&
0.11 & 0.01 & 0.10 & 0.09\\

BERT-large-uncased &
0.15 & 0.02 & 0.15 & 0.10 &&
0.16 & 0.03 & 0.15 & 0.11 &&
0.14 & 0.02 & 0.14 & 0.09 &&
0.18 & 0.03 & 0.18 & 0.12 &&
0.15 & 0.02 & 0.13 & 0.12 \\

BERT-large-cased &
0.09 & 0.01 & 0.08 & 0.06 &&
0.12 & 0.01 & 0.11 & 0.08 &&
0.07 & 0.00 & 0.07 & 0.04 &&
0.12 & 0.01 & 0.11 & 0.07 &&
0.13 & 0.02 & 0.12 & 0.11\\

RoBERTa-base &
0.07 & 0.00 & 0.07 & 0.05 &&
0.04 & 0.00 & 0.04 & 0.03 &&
0.06 & 0.00 & 0.06 & 0.04 &&
0.05 & 0.00 & 0.05 & 0.04 &&
0.10 & 0.01 & 0.09 & 0.08\\

RoBERTa-large &
0.06 & 0.00 & 0.06 & 0.04 &&
0.07 & 0.01 & 0.07 & 0.05 &&
0.07 & 0.00 & 0.07 & 0.04 &&
0.08 & 0.00 & 0.08 & 0.05 && 
0.10 & 0.01 & 0.09 & 0.07 \\

DeBERTa-v3-small &
0.08 & 0.01 & 0.08 & 0.05 &&
0.12 & 0.01 & 0.12 & 0.07 &&
0.07 & 0.00 & 0.06 & 0.03 &&
0.16 & 0.02 & 0.15 & 0.18 &&
0.14 & 0.02 & 0.13 & 0.13\\

DeBERTa-v3-base &
0.02 & 0.00 & 0.02 & 0.02 &&
0.11 & 0.01 & 0.10 & 0.07 &&
0.01 & 0.00 & 0.01 & 0.02 &&
0.11 & 0.01 & 0.11 & 0.07 &&
0.07 & 0.00 & 0.07 & 0.06 \\

DeBERTa-v3-large &
0.07 & 0.00 & 0.07 & 0.05 &&
0.19 & 0.04 & 0.19 & 0.14 &&
0.07 & 0.00 & 0.06 & 0.04 &&
0.18 & 0.03 & 0.18 & 0.13 &&
0.27 & 0.08 & 0.25 & 0.24 \\
\midrule

& 
\multicolumn{24}{c}{{Enron}} \\
\midrule

GPT2 &
0.06 & 0.00 & 0.06 & 0.05 &&
0.05 & 0.00 & 0.05 & 0.04 &&
0.06 & 0.00 & 0.05 & 0.04 &&
0.04 & 0.00 & 0.04 & 0.04 &&
0.14 & 0.01 & 0.12 & 0.12 \\

GPT2-medium &
0.08 & 0.01 & 0.08 & 0.07 &&
0.06 & 0.00 & 0.05 & 0.07 &&
0.07 & 0.00 & 0.07 & 0.07 &&
0.06 & 0.00 & 0.06 & 0.07 &&
0.15 & 0.02 & 0.12 & 0.14 \\

GPT2-large &
0.11 & 0.01 & 0.11 & 0.09 &&
0.07 & 0.00 & 0.07 & 0.07 &&
0.11 & 0.01 & 0.11 & 0.09 &&
0.11 & 0.01 & 0.11 & 0.09 &&
0.16 & 0.02 & 0.13 & 0.14 \\

GPT2-xl &
0.09 & 0.01 & 0.09 & 0.07 &&
0.06 & 0.00 & 0.06 & 0.06 &&
0.09 & 0.01 & 0.09 & 0.08 &&
0.06 & 0.00 & 0.05 & 0.05 &&
0.12 & 0.01 & 0.10 & 0.10 \\

Llama-2-7b &
0.06 & 0.00 & 0.06 & 0.06 &&
0.04 & 0.00 & 0.04 & 0.07 &&
0.04 & 0.00 & 0.04 & 0.06 &&
0.02 & 0.00 & 0.02 & 0.03 &&
0.11 & 0.00 & 0.09 & 0.09 \\

Llama-2-13b &
0.05 & 0.00 & 0.05 & 0.05 &&
0.03 & 0.00 & 0.03 & 0.05 &&
0.03 & 0.00 & 0.03 & 0.05 &&
0.03 & 0.00 & 0.02 & 0.04 &&
0.11 & 0.00 & 0.09 & 0.10 \\

Llama-3-8b &
0.04 & 0.00 & 0.04 & 0.04 &&
0.10 & 0.00 & 0.09 & 0.08 &&
0.04 & 0.00 & 0.04 & 0.04 &&
0.07 & 0.00 & 0.07 & 0.06 &&
0.11 & 0.00 & 0.09 & 0.09 \\

Llama-3.1-8b &
0.04 & 0.00 & 0.04 & 0.04 &&
0.10 & 0.01 & 0.10 & 0.08 &&
0.04 & 0.00 & 0.04 & 0.03 &&
0.07 & 0.00 & 0.07 & 0.06 &&
0.11 & 0.00 & 0.09 & 0.09 \\

Llama-3.2-1b &
0.20 & 0.04 & 0.19 & 0.15 &&
0.22 & 0.04 & 0.21 & 0.16 &&
0.19 & 0.03 & 0.18 & 0.15 &&
0.21 & 0.03 & 0.20 & 0.18 &&
0.17 & 0.02 & 0.14 & 0.15 \\

Llama-3.2-3b &
0.24 & 0.05 & 0.23 & 0.19 &&
0.25 & 0.05 & 0.24 & 0.19 &&
0.23 & 0.04 & 0.21 & 0.18 &&
0.23 & 0.05 & 0.22 & 0.19 &&
0.20 & 0.04 & 0.17 & 0.18 \\

Gemma-2-2b &
0.13 & 0.01 & 0.12 & 0.08 &&
0.22 & 0.05 & 0.22 & 0.16 &&
0.12 & 0.01 & 0.12 & 0.07 &&
0.23 & 0.05 & 0.23 & 0.17 &&
0.22 & 0.05 & 0.19 & 0.21 \\

Gemma-2-9b &
0.12 & 0.01 & 0.12 & 0.09 &&
0.15 & 0.02 & 0.14 & 0.10 &&
0.10 & 0.01 & 0.10 & 0.10 &&
0.17 & 0.03 & 0.17 & 0.15 &&
0.22 & 0.06 & 0.20 & 0.20 \\

\midrule

&
\multicolumn{24}{c}{{Legal}} \\
\midrule

GPT2 &
0.08 & 0.00 & 0.07 & 0.06 &&
0.06 & 0.00 & 0.06 & 0.05 &&
0.06 & 0.00 & 0.06 & 0.04 &&
0.06 & 0.00 & 0.06 & 0.05 &&
0.19 & 0.02 & 0.14 & 0.14 \\

GPT2-medium &
0.08 & 0.00 & 0.08 & 0.07 &&
0.06 & 0.00 & 0.06 & 0.06 &&
0.08 & 0.00 & 0.07 & 0.07 &&
0.07 & 0.00 & 0.06 & 0.07 &&
0.19 & 0.02 & 0.14 & 0.14 \\

GPT2-large &
0.12 & 0.01 & 0.11 & 0.09 &&
0.07 & 0.00 & 0.07 & 0.06 &&
0.14 & 0.01 & 0.13 & 0.10 &&
0.14 & 0.02 & 0.14 & 0.10 &&
0.20 & 0.02 & 0.15 & 0.15 \\

GPT2-xl &
0.10 & 0.01 & 0.09 & 0.07 &&
0.06 & 0.00 & 0.05 & 0.06 &&
0.08 & 0.01 & 0.08 & 0.07 &&
0.06 & 0.00 & 0.05 & 0.06 &&
0.16 & 0.01 & 0.12 & 0.12 \\

Llama-2-7b &
0.07 & 0.00 & 0.07 & 0.06 &&
0.04 & 0.00 & 0.04 & 0.06 &&
0.06 & 0.00 & 0.05 & 0.06 &&
0.03 & 0.00 & 0.02 & 0.04 &&
0.15 & 0.01 & 0.11 & 0.11 \\

Llama-2-13b &
0.06 & 0.00 & 0.06 & 0.05 &&
0.04 & 0.00 & 0.03 & 0.04 &&
0.04 & 0.00 & 0.03 & 0.05 &&
0.03 & 0.00 & 0.03 & 0.04 &&
0.15 & 0.01 & 0.11 & 0.10 \\

Llama-3-8b &
0.04 & 0.00 & 0.04 & 0.04 &&
0.10 & 0.00 & 0.10 & 0.08 &&
0.05 & 0.00 & 0.05 & 0.04 &&
0.07 & 0.00 & 0.07 & 0.06 &&
0.15 & 0.01 & 0.12 & 0.11 \\

Llama-3.1-8b &
0.04 & 0.00 & 0.04 & 0.04 &&
0.10 & 0.01 & 0.09 & 0.07 &&
0.04 & 0.00 & 0.04 & 0.03 &&
0.07 & 0.00 & 0.06 & 0.06 &&
0.15 & 0.01 & 0.12 & 0.10 \\

Llama-3.2-1b &
0.21 & 0.03 & 0.19 & 0.15 &&
0.23 & 0.03 & 0.20 & 0.15 &&
0.21 & 0.03 & 0.19 & 0.15 &&
0.23 & 0.03 & 0.20 & 0.17 &&
0.25 & 0.04 & 0.20 & 0.20 \\

Llama-3.2-3b &
0.26 & 0.05 & 0.24 & 0.20 &&
0.28 & 0.05 & 0.25 & 0.19 &&
0.24 & 0.04 & 0.22 & 0.18 &&
0.24 & 0.04 & 0.22 & 0.18 &&
0.27 & 0.05 & 0.22 & 0.21 \\

Gemma-2-2b &
0.13 & 0.01 & 0.12 & 0.08 &&
0.23 & 0.05 & 0.22 & 0.17 &&
0.10 & 0.01 & 0.10 & 0.06 &&
0.22 & 0.04 & 0.21 & 0.15 &&
0.26 & 0.05 & 0.20 & 0.20 \\

Gemma-2-9b &
0.14 & 0.01 & 0.13 & 0.10 &&
0.19 & 0.02 & 0.17 & 0.12 &&
0.12 & 0.01 & 0.11 & 0.10 &&
0.19 & 0.03 & 0.18 & 0.15 &&
0.25 & 0.04 & 0.20 & 0.19 \\

\midrule

&
\multicolumn{24}{c}{{News}} \\
\midrule

GPT2 &
0.07 & 0.00 & 0.07 & 0.06 &&
0.05 & 0.00 & 0.05 & 0.05 &&
0.07 & 0.00 & 0.07 & 0.05 &&
0.05 & 0.00 & 0.05 & 0.05 &&
0.18 & 0.03 & 0.14 & 0.14 \\

GPT2-medium &
0.09 & 0.01 & 0.08 & 0.07 &&
0.07 & 0.00 & 0.06 & 0.07 &&
0.07 & 0.00 & 0.07 & 0.08 &&
0.07 & 0.00 & 0.06 & 0.08 &&
0.18 & 0.02 & 0.14 & 0.14 \\

GPT2-large &
0.10 & 0.01 & 0.09 & 0.08 &&
0.07 & 0.00 & 0.07 & 0.07 &&
0.10 & 0.01 & 0.09 & 0.08 &&
0.10 & 0.01 & 0.09 & 0.08 &&
0.18 & 0.03 & 0.15 & 0.14 \\

GPT2-xl &
0.08 & 0.01 & 0.08 & 0.07 &&
0.06 & 0.00 & 0.06 & 0.06 &&
0.08 & 0.00 & 0.07 & 0.07 &&
0.05 & 0.00 & 0.05 & 0.05 &&
0.14 & 0.01 & 0.11 & 0.11 \\

Llama-2-7b &
0.05 & 0.00 & 0.05 & 0.05 &&
0.03 & 0.00 & 0.03 & 0.05 &&
0.05 & 0.00 & 0.05 & 0.06 &&
0.02 & 0.00 & 0.02 & 0.04 &&
0.08 & 0.00 & 0.06 & 0.05 \\

Llama-2-13b &
0.05 & 0.00 & 0.04 & 0.04 &&
0.04 & 0.00 & 0.03 & 0.04 &&
0.03 & 0.00 & 0.03 & 0.04 &&
0.03 & 0.00 & 0.03 & 0.03 &&
0.10 & 0.00 & 0.07 & 0.06 \\

Llama-3-8b &
0.03 & 0.00 & 0.03 & 0.04 &&
0.08 & 0.00 & 0.08 & 0.07 &&
0.04 & 0.00 & 0.04 & 0.03 &&
0.06 & 0.00 & 0.06 & 0.06 &&
0.11 & 0.00 & 0.08 & 0.08 \\

Llama-3.1-8b &
0.03 & 0.00 & 0.03 & 0.04 &&
0.09 & 0.01 & 0.09 & 0.07 &&
0.03 & 0.00 & 0.03 & 0.02 &&
0.06 & 0.00 & 0.06 & 0.05 &&
0.12 & 0.01 & 0.09 & 0.09 \\

Llama-3.2-1b &
0.19 & 0.03 & 0.18 & 0.14 &&
0.21 & 0.03 & 0.19 & 0.14 &&
0.18 & 0.03 & 0.17 & 0.14 &&
0.21 & 0.03 & 0.19 & 0.17 &&
0.21 & 0.04 & 0.17 & 0.16 \\

Llama-3.2-3b &
0.23 & 0.04 & 0.21 & 0.17 &&
0.24 & 0.05 & 0.22 & 0.16 &&
0.21 & 0.04 & 0.19 & 0.16 &&
0.21 & 0.04 & 0.21 & 0.17 &&
0.22 & 0.04 & 0.18 & 0.18 \\

Gemma-2-2b &
0.14 & 0.01 & 0.13 & 0.08 &&
0.23 & 0.05 & 0.21 & 0.15 &&
0.13 & 0.01 & 0.12 & 0.07 &&
0.23 & 0.05 & 0.22 & 0.15 &&
0.29 & 0.09 & 0.24 & 0.25 \\

Gemma-2-9b &
0.13 & 0.01 & 0.12 & 0.09 &&
0.17 & 0.02 & 0.16 & 0.11 &&
0.11 & 0.01 & 0.10 & 0.09 &&
0.18 & 0.03 & 0.17 & 0.14 &&
0.29 & 0.09 & 0.25 & 0.24 \\

\bottomrule
\end{tabular}
}

\end{threeparttable}
\end{table*}

    \clearpage

}

\end{document}